%% file: FokkerPlanckTransport.tex
\documentclass[a4paper, twoside]{scrartcl}
\input{Packages}
\input{Macros}

\usepackage{etex}
\usepackage[a4paper, top=88pt, bottom=88pt, left=72pt, right=72pt, headsep=16pt, footskip=28pt]{geometry}
\usepackage{hyperref}
\hypersetup{
	colorlinks,
	linkcolor=blue,
	citecolor=blue,
	urlcolor=blue,
}
\newcommand*{\arXiv}[1]{\bgroup\color{blue}\href{https://arxiv.org/abs/#1}{arXiv:#1}\egroup}
\newcommand*{\doi}[1]{\bgroup\color{blue}\href{https://doi.org/#1}{doi:#1}\egroup}
\newcommand*{\email}[1]{\bgroup\color{blue}\href{mailto:#1}{#1}\egroup}
\renewcommand*{\url}[1]{\bgroup\color{blue}\href{#1}{#1}\egroup}
\usepackage{enumitem, moreenum}
\setlist{nosep}
\usepackage{mleftright} \mleftright
\renewcommand{\qedsymbol}{$\blacksquare$}
\renewenvironment{proof}[1][\proofname]{\noindent{\bfseries\sffamily #1.} }{\hfill\qedsymbol\medskip}
\usepackage[labelfont={sf,bf}]{caption}
\usepackage{scrlayer-scrpage, xhfill}
\automark[section]{section}
\setkomafont{pageheadfoot}{\normalcolor\sffamily}
\setkomafont{pagenumber}{\normalfont\normalsize\sffamily}
\clearpairofpagestyles
\let\oldtitle\title
\renewcommand{\title}[1]{\oldtitle{#1}\newcommand{\theshorttitle}{#1}}
\newcommand{\shorttitle}[1]{\renewcommand{\theshorttitle}{#1}}
\let\oldauthor\author
\renewcommand{\author}[1]{\oldauthor{#1}\newcommand{\theshortauthor}{#1}}
\newcommand{\shortauthor}[1]{\renewcommand{\theshortauthor}{#1}}
\cohead{\xrfill[0.525ex]{0.6pt}~\theshorttitle~\xrfill[0.525ex]{0.6pt}}
\cehead{\xrfill[0.525ex]{0.6pt}~\theshortauthor~\xrfill[0.525ex]{0.6pt}}
\newcommand{\theabstract}[1]{\par\bgroup\noindent\textbf{\textsf{Abstract.}} #1\egroup}
\newcommand{\thekeywords}[1]{\par\smallskip\bgroup\noindent\textbf{\textsf{Keywords.}}\newcommand{\and}{ $\bullet$ } #1\egroup}
\newcommand{\themsc}[1]{\par\smallskip\bgroup\noindent\textbf{\textsf{2010 Mathematics Subject Classification.}}\newcommand{\and}{ $\bullet$ } #1\egroup}

\newcommand*{\affilref}[1]{\ref{affiliation#1}}
\newcommand*{\affiliation}[3]{
	\footnotetext[#1]{\label{affiliation#2}#3}
}

\setlist{topsep=0.3ex, itemsep=0.3ex}

\title{Deterministic Fokker--Planck Transport}
\subtitle{With Applications to Sampling, Variational Inference,
	\\
	Kernel Mean Embeddings \& Sequential Monte Carlo}
\shorttitle{Transporting Quasi-Monte Carlo Points and Sparse Grids}
\author{Ilja~Klebanov\textsuperscript{\affilref{FUB}}}
\shortauthor{Ilja~Klebanov}
\date{\today}

\begin{document}
\maketitle
\affiliation{1}{FUB}{Freie Universit{\"a}t Berlin, Arnimallee 6, 14195 Berlin, Germany (\email{klebanov@zedat.fu-berlin.de})}

\begin{abstract}
	\theabstract{\input{./chunk-abstract.tex}}

	\thekeywords{\input{./chunk-keywords.tex}}
	\themsc{\input{./chunk-msc.tex}}
\end{abstract}

\input{Section_Introduction}
\input{Section_RelatedWork}
\input{Section_Notation}
\input{Section_Fokker}

\input{Section_Approach}
\input{Section_Applications}

\input{Section_VI}

\input{Section_Direct_Sampling}
\input{Section_Transport_Sampling}
\input{Section_Convolution_Sampling}

\input{Section_Outbedding}

\input{Section_Annealing}
\input{Section_Conclusion}

\section*{Acknowledgements}
\addcontentsline{toc}{section}{Acknowledgements}

This research was funded by the Deutsche Forschungsgemeinschaft (DFG, German Research Foundation) under Germany's Excellence Strategy (EXC-2046/1, project 390685689) through projects EF1-10 and EF1-19 of the Berlin Mathematics Research Center MATH+.
The author would like to thank P\'eter Koltai, Nicolas Perkowski, Sebastian Reich, Claudia Schillings, Ingmar Schuster, Tim Sullivan, and Simon Weissmann for many inspiring discussions and helpful suggestions.

\appendix
\input{Section_AppendixProofs}

\bibliographystyle{abbrvnat}
\bibliography{myBibliography}
\addcontentsline{toc}{section}{References}

\end{document}

%% file: Packages.tex
\usepackage{amsfonts,amssymb,amsmath,amsthm,amstext,amssymb,mathtools}
\usepackage{subcaption,float,color,xcolor,graphicx,enumitem}
\usepackage{dsfont}  
\usepackage[normalem]{ulem} 
\usepackage[authoryear,sort,round]{natbib}  
\usepackage{fullpage} 
\usepackage{hyperref,nicefrac}
\usepackage[noabbrev,capitalise,nosort,nameinlink]{cleveref}

\crefname{assumption}{Assumption}{Assumptions}
\crefname{openproblem}{Open Problem}{Open Problem}
\crefname{notation}{Notation}{Notation}

\usepackage{adjustbox}
\usepackage{tikz}
\usetikzlibrary{bayesnet,er,trees,shapes.symbols,mindmap,arrows,decorations,shapes.misc,shapes.arrows,chains,matrix,positioning,scopes,decorations.pathmorphing,patterns}
\usepackage{tikz-cd} 

\usepackage{array}
\newcolumntype{x}[1]{>{\centering\arraybackslash\hspace{0pt}}p{#1}}
\usepackage{booktabs,colortbl,xcolor,xfrac}
\usepackage{multirow}
\usepackage{multicol}

\usepackage{pifont}

\usepackage{soul}

%% file: Macros.tex
\newcommand*{\cX}{\mathcal{X}}
\newcommand*{\cY}{\mathcal{Y}}
\newcommand*{\cH}{\mathcal{H}}

\newcommand*{\cK}{\mathcal{K}}

\newcommand*{\cP}{\mathcal{P}}

\newcommand*{\cO}{\mathcal{O}}
\newcommand*{\cE}{\mathcal{E}}

\newcommand*{\bz}{\mathbf{z}}
\newcommand*{\bw}{\mathbf{w}}
\newcommand*{\R}{\mathbb{R}}

\newcommand{\bsw}{{\boldsymbol{w}}}

\newcommand{\bsX}{{\boldsymbol{X}}}

\newcommand*{\init}{{\textup{ref}}}
\newcommand*{\fin}{{\textup{tar}}}

\newcommand*{\den}{\rho}
\newcommand*{\vel}{v}

\newcommand*{\FP}{{\textup{FP}}}
\newcommand*{\SVGD}{{\textup{SVGD}}}

\newcommand*{\trd}{\rho_{\textup{tar}}}
\newcommand*{\simpd}{\rho_{\textup{ref}}}
\newcommand*{\Ball}[2]{B_{#2}(#1)}

\DeclareMathOperator{\diver}{div}
\DeclareMathOperator{\dkl}{D_{KL}}
\DeclareMathOperator{\KDE}{KDE}
\DeclareMathOperator{\MMD}{MMD}

\DeclareMathOperator{\Unif}{Unif}

\usepackage{mathabx} 
\mathchardef\ordinarycolon\mathcode`\:
\mathcode`\:=\string"8000
\begingroup \catcode`\:=\active
\gdef:{\mathrel{\mathop\ordinarycolon}}
\endgroup

\definecolor{ourblue}{rgb}{0,0,0.7}
\definecolor{ouryellow}{rgb}{1.0, 0.88, 0.21}

\newcommand{\hlc}[2][yellow]{{%
		\colorlet{foo}{#1}%
		\sethlcolor{foo}\hl{#2}}%
}

\newcommand*{\bR}{\mathbb{R}}
\newcommand*{\bN}{\mathbb{N}}

\newcommand*{\bE}{\mathbb{E}}

\newcommand*{\bP}{\mathbb{P}}

\newcommand*{\cN}{\mathcal{N}}

\newcommand*{\iid}{\textup{i.i.d.}}

\makeatletter
\newcommand*\quark{\mathpalette\quark@{.5}}
\newcommand*\quark@[2]{\mathbin{\vcenter{\hbox{\scalebox{#2}{$\; \m@th#1\bullet \;$}}}}}
\makeatother

\makeatletter
\newcommand{\mylabel}[2]{#2\def\@currentlabel{#2}\label{#1}}
\makeatother

\newcommand*{\rd}{\mathrm{d}}

\DeclareMathOperator{\Id}{Id}

\DeclareMathOperator*{\argmin}{arg\,min}

\definecolor{mygreen}{rgb}{0,0.6,0}
\definecolor{darkgreen}{rgb}{0,0.4,0}

\newcommand*{\todo}[1]{\bgroup\color{red}TODO:~#1\egroup}
\newcommand*{\tjs}[1]{\bgroup\color{olive}TJS:~#1\egroup}
\newcommand*{\bs}[1]{\bgroup\color{orange}BS:~#1\egroup}
\newcommand*{\ik}[1]{\bgroup\color{cyan}IK:~#1\egroup}

\DeclarePairedDelimiter\abs{\lvert}{\rvert}%
\DeclarePairedDelimiter\mynorm{\lVert}{\rVert}%
\makeatletter
\let\oldabs\abs
\def\abs{\@ifstar{\oldabs}{\oldabs*}}
\let\oldnorm\mynorm
\def\mynorm{\@ifstar{\oldnorm}{\oldnorm*}}
\makeatother

\newcommand*{\defeq}{\coloneqq}

\theoremstyle{plain}
\newtheorem{theorem}{\sffamily Theorem}[section]
\newtheorem{proposition}[theorem]{\sffamily Proposition}

\theoremstyle{definition}
\newtheorem{definition}[theorem]{\sffamily Definition}

\newtheorem{algorithm}[theorem]{\sffamily Algorithm}

\newtheorem{remark}[theorem]{\sffamily Remark}

\newcommand{\absval}[1]{\lvert #1 \rvert}
\newcommand{\innerprod}[2]{\langle #1 , #2 \rangle}
\newcommand{\norm}[1]{\lVert #1 \rVert}

\newcommand{\Absval}[1]{\left\vert #1 \right\vert}

\numberwithin{equation}{section}
\numberwithin{figure}{section}
\numberwithin{table}{section}

%% file: chunk-abstract.tex

The Fokker--Planck equation can be reformulated as a continuity equation, which naturally suggests using the associated velocity field in particle flow methods.
While the resulting \emph{probability flow ODE} offers appealing properties---such as defining a gradient flow of the Kullback--Leibler divergence between the current and target densities with respect to the 2-Wasserstein distance---it relies on evaluating the current probability density, which is intractable in most practical applications.
By closely examining the drawbacks of approximating this density via kernel density estimation, we uncover opportunities to turn these limitations into advantages in contexts such as variational inference, kernel mean embeddings, and sequential Monte Carlo.


%
%
%

%% file: chunk-keywords.tex
sampling%
\and
Monte Carlo methods%
\and
quasi-Monte Carlo%
\and
Bayesian inference%
\and
super-root-$n$ convergence%
\and
kernel density estimation%
\and
kernel mean embedding%
\and
variational inference%

%% file: chunk-msc.tex
65C05
\and%
65D32
\and%
11K36 
\and
46E22
\and
62G05

%% file: Section_Introduction.tex

\section{Introduction}
\label{section:Introduction}

Approximating a target probability density $\trd \colon \bR^{d} \to \bR_{\geq 0}$ is a fundamental challenge in Bayesian statistics and machine learning. The difficulties involved typically arise from high dimensionality, multimodality, and the fact that the normalizing constant of $\trd$ is often unknown. While Markov chain Monte Carlo (MCMC) methods \citep{hastings1970monte,meyn2009markov,robert2004monte} dominate the field of Bayesian inference, several powerful Monte Carlo methods employ a technique known as \emph{bridging}.
This approach connects an easily sampled reference probability distribution $\simpd$ with the target $\trd$ via a family of intermediate densities $(\rho_t)_{t \in [0,\tau]}$, $\tau \in [1, \infty]$, with $\rho_0 = \simpd$ and $\rho_\tau = \trd$ (in discrete cases, a finite sequence $(\rho_t)_{t = 0,\dots,\tau}$ is defined).

The intermediate densities may be predefined, as in sequential Monte Carlo methods \citep{Doucet2001SMC,delmoral2006smc,Chopin2020SMC}, annealed importance sampling \citep{neal2001annealed}, and path sampling \citep{gelman1998simulating}, or they can arise from perturbations or drifts that bring $\rho_t$ closer to $\trd$ over time (often minimizing the Kullback--Leibler divergence). This is the case for particle flow methods and particle-based variational inference \citep{liu2019understanding,wang2019accelerated,chen2018unified,chen2018stein,chen2019stein,liu2016stein}, as well as normalizing flows \citep{rezende2015variational}, neural ordinary differential equations \citep{chen2018neural}, and the methodology developed by \citet{tabak2010density,tabak2013family}.

We begin by focusing on particle flow methods, where the density $\den_t$ is approximated by an ensemble of particles $\bsX(t) = (X_1(t), \dots, X_J(t)) \in (\bR^d)^J$, that is, by the discrete distribution $J^{-1} \sum_{j=1}^{J} \delta_{X_j(t)}$, with $\delta_x$ denoting the Dirac measure at $x$.
Motivated by a reformulation of the Fokker--Planck equation (FPE) as a continuity equation, we introduce a deterministic particle dynamics known as the \emph{probability flow ODE} \citep{boffi2023probabilityflowsolutionfokkerplanck}. This ODE drives $\simpd$ exponentially toward $\trd$, but the velocity field governing this dynamics, $\vel_t^{\FP} = \nabla \log \frac{\trd}{\rho_t}$, requires knowledge of the current probability density $\rho_t$. Unfortunately, computing $\rho_t$ is typically infeasible in high dimensions, rendering direct solutions of the FPE impractical.
It is important to note that $\vel_t^{\FP}$ depends only on the gradient of $\log \trd$, meaning that $\trd$ does not need to be normalized. This feature is particularly valuable in Bayesian inference, where the normalizing constant of the posterior is often unknown.

Alternative approaches have been proposed in the literature \citep{liu2016stein,reich2015probabilistic,peyre2019computational}, one of which is to approximate $\rho_t$ via kernel density estimation (KDE):
\begin{equation}
	\label{equ:definition_KDE}
	\KDE_{\kappa}^{h}[\bsX] \defeq \frac{1}{J} \sum_{j=1}^{J} \kappa^{h}(x - X_j),
	\qquad
	\kappa^{h}(x) \defeq h^{-d} \kappa(h^{-1} x),
	\qquad
	\bsX = (X_1, \dots, X_J) \in (\bR^d)^J,
\end{equation}
where $\bsX = \bsX(t)$ are the current particle positions \citep{wang2019accelerated,liu2019understanding}. Here, $\kappa^{h} \in L^1(\bR^d; \bR_{\geq 0})$ is a probability density function acting as the kernel, and $h > 0$ is the bandwidth parameter.

While previous studies \citep{wang2019accelerated,liu2019understanding} emphasize bandwidth selection and the smoothing effect of KDE, they do not address the impact of this effect on the long-term particle positions, $\bsX(\infty) \coloneqq \lim_{t \to \infty} \bsX(t)$, which we term \emph{KDE points}. Specifically, these KDE points do not follow the desired target distribution $\trd$ but are instead overly concentrated, effectively sampling from a density $\check{\den}^{h}$ that satisfies $\check{\den}^{h} \ast \kappa^{h} \approx \trd$.

Contrary to initial expectations, these effects are not merely drawbacks but, in fact, reveal fundamental advantages that form the central contribution of this work:
\vspace{1ex}


\begin{enumerate}[label=(\Alph*)]
	\item
	\label{item:KDE_points_KDE_property}
	\colorbox{ourblue!8}{In the long term, the points $\bsX(t)$ arrange themselves such that
		$\KDE_{\kappa}^{h}[\bsX (\infty)] \approx \trd$.}
	\\[1.5ex]
	\textbf{Advantages:}
	\begin{itemize}
		\item
		Similar to variational inference \citep{blei2017variational}, having an analytic expression for the approximate target density enables addressing a wide range of practical problems, including model selection and prediction tasks.
		\item
		The approximate target $\KDE_{\kappa}^{h}[\bsX (\infty)]$ allows for efficient sampling, especially when simple kernels, commonly used in practice, are employed. Additionally, importance reweighting (\citealt[Section~5.7]{rubinstein2016simulation}; \citealt[Section~3.3]{robert2004monte}) can be applied to account for any approximation error.
		\item
		Since $\KDE_{\kappa}^{h}[\bsX (\infty)]$ defines a mixture distribution, more advanced sampling techniques, such as \emph{stratified sampling}, can be employed to reduce the variance in Monte Carlo estimators \citealt[Section~5.5]{rubinstein2016simulation}. Other techniques, such as (transported) Quasi Monte Carlo points or sparse grids, can be utilized to accelerate convergence rates \citep{Cui2023quasimonte,klebanov2023transporting}.
	\end{itemize}
	\vspace{1ex}
	\item
	\label{item:KDE_points_deconvolution_property}
	\colorbox{ouryellow!16}{The points $\bsX (\infty)$ are distributed according to $\check{\den}^{h}$, where $\check{\den}^{h} \ast \kappa^{h} \approx \trd$.}
	\\[1.5ex]
	\textbf{Advantages:}
	\begin{itemize}
		\item
		In certain applications, such as the inversion of kernel mean embeddings, samples from $\check{\den}^{h}$ are required rather than from $\trd$. Thus, by a fortunate coincidence, this method offers an elegant solution to an otherwise complex problem as a natural byproduct, with a notable application to resampling in sequential Monte Carlo methods \citep{doucet2000sequential,Doucet2001SMC,delmoral2006smc}.
		\item
		Empirical results indicate that the KDE points $\bsX (\infty)$, which are strongly correlated through the flow dynamics, tend to be evenly spaced and exhibit a super-root-$n$ convergence rate\footnote{Super-root-$n$ convergence refers to a rate faster than $O(n^{-1/2})$, where $n$ is the sample size, thus surpassing the speed of conventional Monte Carlo methods.}
		in corresponding Monte Carlo estimators for $\check{\den}^{h}$.
	\end{itemize}
\end{enumerate}

In summary, we shift our focus from directly solving the probability flow ODE to approximating the target distribution $\trd$ using the mixture distribution $\KDE_{\kappa}^{h}[\bsX(\infty)]$, which aligns with the principles of variational Bayes methods \citep{blei2017variational}. These insights can be applied in two ways, corresponding to observations \ref{item:KDE_points_KDE_property} and \ref{item:KDE_points_deconvolution_property} above and summarized in \Cref{table:application_areas}.
Specifically, one can either use the approximation $\KDE_{\kappa}^{h}[\bsX (\infty)] \approx \trd$ or directly leverage the KDE points $\bsX (\infty)$ themselves.

This paper is structured as follows. \Cref{section:RelatedWork} reviews related work, while \Cref{section:Notation} introduces the necessary notation.
\Cref{section:TransportAndFokkerPlanck} recaps well-established results on the continuity equation and FPE, identifying the velocity field $\vel_t^{\FP} = \nabla \log \frac{\trd}{\rho_t}$ as a deterministic dynamics that mimics the FPE.
\Cref{section:Approach} explores the effect of approximating $\rho_t$, which enters the velocity field's formulation, using a KDE based on the current particle positions, leading to the formulation of a mixture distribution approximation for $\trd$.
We explore several applications in \Cref{section:Applications} and demonstrate how these KDE effects can be beneficial in two key areas:
one where the KDE serves as an approximation of the target distribution $\trd$, and another where the KDE points $\bsX (\infty)$ are used directly.
\Cref{section:Annealing} explores how a simulated annealing-like method can improve the performance for multimodal target distributions. Finally, \Cref{section:Conclusion} concludes the paper, outlines unresolved problems, while the proofs of two key theorems are given in \Cref{appendix:Proofs}.

%% file: Section_RelatedWork.tex

\section{Related Work}
\label{section:RelatedWork}

This work is closely related to particle-based variational inference methods. The velocity field considered here, $\vel_t^{\FP} = \nabla \log \frac{\trd}{\rho_t}$, has been explored in similar contexts by
\citet{liu2016stein,reich2015probabilistic,peyre2019computational,pathiraja2019discrete,Reich2021FPparticles} and others.
In our approach, the unknown density $\rho_t$ in the denominator is approximated via kernel density estimation (KDE) based on the particle ensemble, as seen in the work of \citet{wang2019accelerated} and \citet{liu2019understanding}.
However, KDE introduces its own challenges, which we address in \Cref{section:Approach}, ultimately turning these challenges to our advantage.

Unlike the aforementioned approaches, where the target density $\trd$ is approximated by the particles themselves in a Monte Carlo sense, our method focuses on approximating $\trd$ through the kernel density estimate $\KDE_{\kappa}^{h}[\bsX (\infty)]$, based on the final positions of the particles.
This connects our approach to variational inference frameworks \citep{jordan1999introduction,attias1999variational,blei2017variational} and shares similarities with certain hybrids of MCMC and variational Bayesian methods \citep{martino2017layered,schuster2021MCIS}.
This perspective opens up several possibilities for leveraging these points for sampling from the target distribution, including the use of transported quasi-Monte Carlo points \citep{Cui2023quasimonte,klebanov2023transporting}. 
Our approach identifies the ``deconvolved'' density $\check{\den}^{h}$, which describes the distribution of the final particles in a certain sense. This setup---sampling from $\check{\den}^{h}$ while evaluating $\trd$ and its gradient---bears resemblance to kernel herding \citep{chen2010super,Bach2012HerdingConditionalGradient,lacoste2015sequential} and Sequential Bayesian Quadrature \citep{huszar2012herding}, to which we compare our method.

A direct application of this method is the inversion of kernel mean embeddings \citep{smola2007embedding,berlinet2004rkhs,muandet2017kernel}, a crucial step in numerous machine learning tasks, as well as an efficient resampling strategy for sequential Monte Carlo methods \citep{doucet2000sequential,Doucet2001SMC,delmoral2006smc,Chopin2020SMC}.

%% file: Section_Notation.tex

\section{Preliminaries and Notation}
\label{section:Notation}

Throughout this paper, we assume that all probability densities $\rho, \rho_{t}, \trd, \kappa \in C^{\infty}(\bR^{d};\bR_{>0})$ are strictly positive and smooth, with both the densities and their gradients decaying sufficiently fast. This ensures that any velocity field of the form $v = \nabla \log \frac{\rho_{1}}{\rho_{2}}$ is well-defined and smooth, and that the corresponding continuity equation holds in a strong sense with a smooth solution at all times. While some results, particularly the continuity equation \eqref{equ:TransportEquation} in \Cref{section:TransportAndFokkerPlanck}, can be formulated in a weak sense within a more general setting \citep{ambrosio2008transport, ambrosio2008gradient, villani2003topics}, such an approach would introduce additional technical and notational complexity that lies outside the scope of this paper.

For simplicity, we avoid introducing the corresponding probability distribution $\bP_{\rho}$ for each density $\rho$, and instead, use the density symbol $\rho$ itself in place of $\bP_{\rho}$. For example, $X \sim \rho$ means $X \sim \bP_{\rho}$, and $T_{\#} \rho$ denotes the pushforward of $\bP_{\rho}$ under the transport map $T$.

We use the following notational conventions: For any $x \in \bR^{d}$, the Euclidean norm is denoted by $\norm{x} \defeq \norm{x}_{2}$. The closed ball of radius $\varepsilon \geq 0$ centered at $x \in \bR^{d}$ is denoted by $\Ball{x}{\varepsilon} \defeq \{ y \in \bR^{d} \mid \norm{x - y} \leq \varepsilon \}$.
For $d, d' \in \bN$ and a probability density $\rho \in L^{1}(\bR^{d})$, we define the weighted Lebesgue space $L_{\rho}^{2} = L_{\rho}^{2}(\bR^{d}; \bR^{d'})$ as the space of measurable functions $f\colon \bR^{d} \to \bR^{d'}$ such that $\norm{f}_{L_{\rho}^{2}} < \infty$, where the corresponding norm and inner product are defined by
\[
\norm{f}_{L_{\rho}^{2}} \defeq \innerprod{f}{f}_{L_{\rho}^{2}}^{1/2},
\quad
\innerprod{f}{g}_{L_{\rho}^{2}} \defeq \int_{\bR^{d}} \rho(x) f(x)^{\top} g(x) \, \rd x.
\]
The Kullback–Leibler divergence between two strictly positive and continuous probability densities $\rho_{1}, \rho_{2} \in C(\bR^{d}; \bR_{>0})$ is defined as
\[
\dkl(\rho_{1} \| \rho_{2}) \defeq \int_{\bR^{d}} \rho_{1}(x) \log\frac{\rho_{1}(x)}{\rho_{2}(x)} \, \mathrm dx \in [0, \infty].
\]
The convolution $f \ast \kappa$ and cross-correlation $f \star \kappa$ of two functions $f, \kappa\colon \bR^{d} \to \bR$ are given by
\[
f \ast \kappa \defeq \int_{\bR^{d}} f(x) \, \kappa(\quark - x) \, \mathrm dx,
\quad
f \star \kappa \defeq \int_{\bR^{d}} f(x) \, \kappa(x - \quark) \, \mathrm dx,
\]
whenever these expressions are defined. In many practical applications, $\kappa$ is symmetric, in which case the two operations coincide.
Finally, the kernel density estimate $\KDE_{\kappa}^{h}[\bsX]$ with respect to the points $\bsX = (X_1, \dots, X_J) \in (\bR^{d})^J$, the kernel $\kappa \in L^{1}(\bR^{d}; \bR_{\geq 0})$, and the bandwidth $h > 0$, is defined as in \eqref{equ:definition_KDE}.

%% file: Section_Fokker.tex

\section{Continuity Equation and Fokker--Planck Equation}
\label{section:TransportAndFokkerPlanck}

It is well-known (\citealt[Proposition 2.1]{ambrosio2008transport}; \citealt[Proposition 8.1.8]{ambrosio2008gradient}; \citealt[Theorem 5.34]{villani2003topics}) that, if the dynamics of a particle with random initial position $X_\init$ is governed by the ordinary differential equation (ODE)
\begin{equation}
	\label{equ:ODE}
	\mathrm d X(t) = \vel_t(X(t))\, \mathrm dt,
	\qquad
	X(0) = X_\init \sim \den_\init \in C^{\infty}(\bR^{d}),
\end{equation}
where $\vel_t \in C^{\infty}(\R^d; \R^d)$ is a smooth, time-dependent velocity field,
the evolution of the corresponding probability density $\den_t$ of $X(t)$ is described by the continuity equation:
\begin{equation}
	\label{equ:TransportEquation}
	\partial_t \den_t
	=
	-\diver(\den_t \vel_t),
	\qquad
	\den_0 = \den_\init.
\end{equation}

Similarly, if the particle's motion is governed by an It{\^o} stochastic differential equation (SDE), often referred to as overdamped Langevin dynamics, of the form
\begin{equation}
	\label{equ:SDE}
	\mathrm d Y(t)
	=
	\nabla \log \trd (Y(t))\, \mathrm dt + \sqrt{2}\, \mathrm dW(t),
	\qquad
	Y(0) = Y_\init \sim \den_\init,
\end{equation}
where $\trd \in C^{\infty}(\bR^{d}; \bR_{>0})$ is a strictly positive target probability density,\footnote{The SDE \eqref{equ:SDE} is often expressed using the potential $\Psi \defeq - \log \trd$. However, we opt for the current form to reduce notational complexity, as $\trd$ is the stationary and limiting density of the dynamics.} and $W(t)$ is a $d$-dimensional Wiener process (diffusion), the evolution of the corresponding density $\den_t$ is governed by the Fokker-Planck equation (FPE; \citealt{risken1989fokker}; \citealt[Section 4]{pavliotis2014stochastic}):
\begin{equation}
	\label{equ:FokkerPlanckEquation}
	\partial_t \den_t
	=
	-\diver(\den_t \nabla\log \trd) + \Delta \den_t,
	\qquad
	\den_0 = \den_\init.
\end{equation}
It is well-known \citep[Theorem 4.4]{pavliotis2014stochastic} that, under fairly general conditions, the solution $\den_{t}$ of \eqref{equ:FokkerPlanckEquation} converges exponentially fast to the unique stationary density $\trd$.

One application of this result is a straightforward sampling technique: given a stationary density $\trd$, one can initialize samples from a smooth, easily sampled density $\den_\init$ and let them evolve according to the SDE \eqref{equ:SDE} for a ``reasonably long'' period of time. 
The resulting points will then be approximately $\trd$-distributed.

However, the time required for this process can be prohibitively long, especially when $\trd$ is multimodal, rendering this approach impractical for many real-world applications.

Interestingly, the Fokker-Planck equation \eqref{equ:FokkerPlanckEquation} can be rewritten as a continuity equation:
\begin{equation}
	\label{equ:FokkerAsContinuity}
	\partial_t \den_t
	=
	-\diver(\den_t \nabla \log \trd - \nabla \den_t)
	=
	-\diver(\den_t \vel_t^{\FP}),
	\qquad
	\vel_t^{\FP}
	\defeq
	\nabla \log \frac{\trd}{\rho_t},
\end{equation}
which leads to the following almost trivial result:
\begin{proposition}
	Let $\simpd, \trd \in C^{\infty}(\bR^{d}; \bR_{>0})$ be smooth, strictly positive probability densities. Then the density evolutions $(\rho_t)_{t\in[0,T]}$ of the SDE \eqref{equ:SDE} and the ODE \eqref{equ:ODE} with $v_t = v_t^{\FP} = \nabla \log \frac{\trd}{\rho_t}$ are identical.
\end{proposition}

\begin{proof}
	The claim follows directly from \eqref{equ:TransportEquation}, \eqref{equ:FokkerPlanckEquation}, and \eqref{equ:FokkerAsContinuity}.
\end{proof}

\begin{figure}[t]
	\centering
	\begin{subfigure}[b]{0.49\textwidth}
		\centering
		\includegraphics[width=\textwidth]{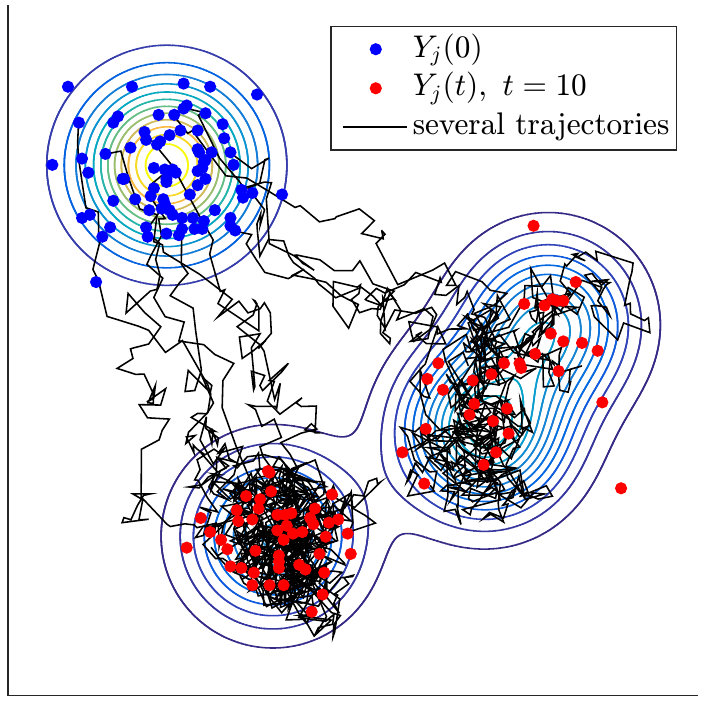}
	\end{subfigure}
	\hfill
	\begin{subfigure}[b]{0.49\textwidth}
		\centering
		\includegraphics[width=\textwidth]{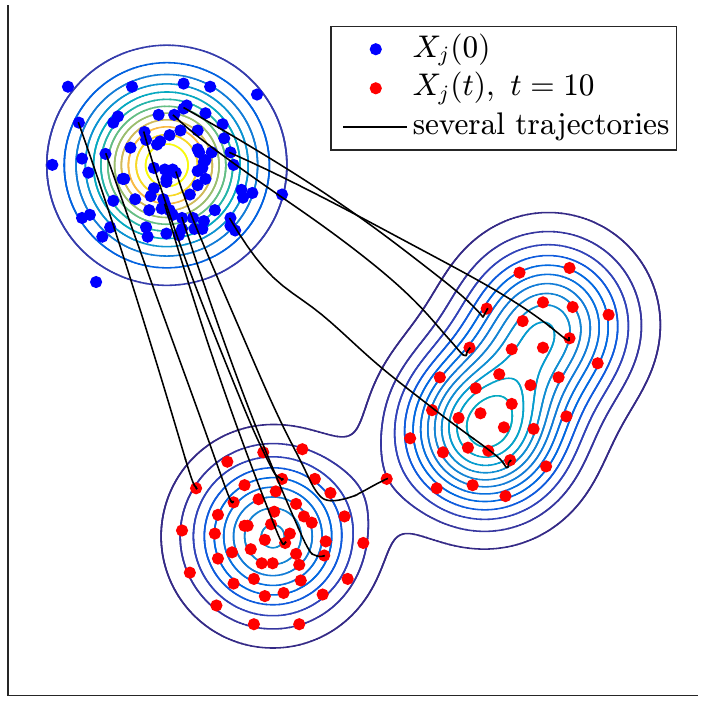}
	\end{subfigure}
	\caption{
		$\den_\init$-distributed samples (standard Gaussian) are transported to $\den_\fin$-distributed samples (a mixture of three Gaussian densities) by the SDE \eqref{equ:SDE} using the Euler--Maruyama discretization (left) and by the ODE \eqref{equ:ODE} with $\vel_t = \vel_t^{\FP}$ (right). 
		The densities of both methods coincide analytically for each $t \ge 0$. 
		Ten trajectories are shown in black. The density estimate required for evaluating $\vel_t^{\FP}$ was performed using kernel density estimation based on the current samples at each time step. 
		The implications of this approximation are discussed in detail in \Cref{section:Approach}. 
		Notably, as observed in the right plot, the final points are too ``concentrated'' to represent the target density $\trd$ and exhibit more regularity than independent samples, particularly near the outer regions.
	}
	\label{fig:SDEvsODE}
\end{figure}

\begin{remark}
	It is important to note that while the evolution of the densities is identical, the trajectories of the SDE \eqref{equ:SDE} and the ODE \eqref{equ:ODE} are fundamentally different. 
	In particular, the trajectories of the ODE will be significantly smoother than those of the SDE due to the absence of the ``dithering'' effect induced by the Wiener process $W(t)$, as illustrated in Figure \ref{fig:SDEvsODE}.
\end{remark}

From a numerical perspective, the ODE offers the advantage of allowing for considerably larger time steps when approximating the trajectories. 
However, the careful reader will notice that the velocity field $\vel_t^{\FP}$ constructed in \eqref{equ:FokkerAsContinuity} depends on the computation of $\den_t$ and its gradient, both of which are challenging and computationally expensive.
Directly solving the continuity equation \eqref{equ:TransportEquation} or the Fokker-Planck equation \eqref{equ:FokkerPlanckEquation} alongside the ODE is impractical for most real-world applications. 
This issue can be circumvented by running a large ensemble of trajectories $X_{j}(t)$, where $X_{j}(0) \stackrel{\iid}{\sim} \den_\init$, using \eqref{equ:ODE}, and then approximating the density $\den_t$ from the current particle positions.

However, estimating densities from samples, especially in high-dimensional settings, introduces its own set of challenges. 
Furthermore, this procedure induces correlations between the initially independent particles $X_{j}(t)$, complicating the analysis. 
This phenomenon often manifests as particle repulsion, as seen in Figure \ref{fig:SDEvsODE} (right), where the final particle positions exhibit a somewhat regular structure.
While \citet[Section~5.6]{reich2015probabilistic} suggest using Gaussian densities to approximate $\vel_t^{\FP}$, \citet[Remark~9.12]{peyre2019computational} propose using distances to nearest neighbors, and \citet{wang2019accelerated} and \citet{liu2019understanding} recommend kernel density estimation (KDE). 
However, the smoothing effect of KDE is so pronounced that the estimated velocity field can no longer be considered a reliable approximation of $\vel_t^{\FP}$.
We explore these issues and their consequences in detail in \Cref{section:Approach}, where our focus shifts from approximating the Fokker-Planck dynamics to approximating $\trd$ using a mixture distributions.

Before continuing with this discussion, we first present some theoretical properties of the Fokker-Planck equation in the following subsection.

\subsection{FPE as a gradient flow of the Kullback--Leibler divergence}
\label{section:FPE_as_gradient_flow}

In the seminal paper by \citet{jordan1998variational}, the authors demonstrate that the Fokker--Planck equation describes a gradient flow of the Kullback--Leibler divergence\footnote{In the original paper, the free energy functional $F(\rho)=\int \Psi(x)\, \rho(x)\, \mathrm dx + \int\rho(x)\, \log\rho(x)\, \mathrm dx$ is used. This coincides with $\dkl (\rho\|\trd)$ when $\trd \propto \exp(-\Psi)$.} $\dkl (\rho_{t}\|\trd)$ between the current density $\rho_{t}$ and the target density $\trd$, with respect to the 2-Wasserstein distance
\begin{equation}
	\label{equ:WassersteinDistance}
	W_2(\rho_1,\rho_2) = \min \{ C_{\rho_1}(T) \mid T\colon\bR^d\to\bR^d,\, T_\#\rho_1 = \rho_2 \},
	\qquad
	C_{\rho}(T)
	:=
	\| T - \Id \|_{L^2_\rho},
\end{equation}
where $\rho_1,\rho_2$ are probability densities on $\bR^d$.
This result provided a new impetus for the development of the theory of gradient flows and geodesic flows in spaces of probability measures \citep{ambrosio2008gradient,villani2003topics,santambrogio2015optimal}, thereby linking optimal transport, partial differential equations, and Riemannian geometry.

In light of \eqref{equ:FokkerAsContinuity}, this implies that the velocity field $v_{t} = \vel_t^{\FP} = \nabla\log \frac{\trd}{\rho_t}$ in the ODE \eqref{equ:ODE} realizes a gradient flow in this sense, as expressed in the following proposition.\footnote{Although this seems to be a well-established result, the author is not aware of a specific reference for this statement.}

\begin{proposition}
	\label{prop:KL_change}
	Let $\rho_t, \trd \in C^{\infty}(\R^d;\bR_{>0})$, $t \geq 0$, be strictly positive and smooth probability densities, and let $(\rho_{t})_{t \geq 0}$ satisfy the continuity equation \eqref{equ:TransportEquation} for a family of vector fields $v_{t} \in C^{\infty}(\R^d; \bR^d)$. 
	Assume that $\rho_{t}v_{t}\log\frac{\rho_{t}}{\trd}$ decays sufficiently fast:
	\[
	\rho_{t}(x) \, \norm{v_{t}(x)} \, \Absval{\log\frac{\rho_{t}(x)}{\trd(x)}} = o(\norm{x}^{1-d}).
	\]
	Then, for any $t \geq 0$,
	\[
	\tfrac{\mathrm d}{\mathrm dt} \dkl(\rho_t \| \trd)
	=
	- \innerprod{v_{t}}{\vel_t^{\FP}}_{L^2_{\rho_t}},
	\qquad
	\vel_t^{\FP}
	=
	\nabla\log \frac{\trd}{\rho_t}.
	\]
\end{proposition}

\begin{proof}
	Note that $\int_{\bR^{d}} \partial_{t} \rho_{t} = 0$ since $\rho_{t}$ is a probability density for each $t \geq 0$.
	By invoking \eqref{equ:TransportEquation} and integrating by parts, we obtain
	\[
	\tfrac{\mathrm d}{\mathrm dt} \dkl(\rho_t \| \trd)
	=
	-\int_{\bR^{d}} \diver(\rho_t v_{t}) \log\frac{\rho_t}{\trd}
	=
	-\int_{\bR^{d}} \rho_t \, v_{t}^{\intercal} \nabla \log\frac{\trd}{\rho_t}
	=
	- \innerprod{v_{t}}{\vel_t^{\FP}}_{L^2_{\rho_t}}.
	\]
\end{proof}

%% file: Section_Approach.tex
\section{The Effects of Kernel Density Estimation}
\label{section:Approach}

\citet{wang2019accelerated} and \citet{liu2019understanding} propose approximating the velocity field $v_{t} = v_{t}^{\FP}$ in the ODE \eqref{equ:ODE} using a kernel density estimate (KDE; \citealt{parzen1962estimation, silverman1986density, scott2015multivariate}) based on the current samples $\bsX(t) = (X_1(t), \dots, X_J(t)) \in (\bR^d)^J$. 
These samples are themselves propagated by the ODE, resulting in the following dynamics:
\begin{equation}
	\label{equ:ODE_with_KDE}
	\dot X_j(t)
	=
	\hat v_t^{h}(X_j(t)),
	\qquad
	\hat{\vel}_{t}^{h}
	\defeq
	\nabla \log \frac{\den_{\fin}}{\hat\den_t^{h}},
	\qquad
	\hat{\rho}_{t}^{h}
	\defeq
	\KDE_{\kappa}^{h}[\bsX (t)].
\end{equation}
This creates an interplay where:
\begin{itemize}
	\item The current samples are used to estimate $\hat{\rho}_{t}^{h}$, an approximation of the current probability density $\rho_t$,
	\item This estimate $\hat{\rho}_{t}^{h}$ is then employed to approximate the velocity field, which further propagates the samples.
\end{itemize}
This approach avoids the need to solve the high-dimensional PDEs \eqref{equ:TransportEquation} or \eqref{equ:FokkerPlanckEquation}, which is computationally infeasible in most practical situations.

However, it is important to note that KDE is not a consistent estimator in the strict sense, as it introduces a smoothing effect. 
As the sample size increases ($J \to \infty$), the KDE does not converge to the true density $\rho$, but rather to a smoothed version of it. 
Specifically, if $X_1,\dots,X_J \stackrel{\textup{i.i.d.}}{\sim} \rho$, then by the law of large numbers, we observe the pointwise convergence:
\begin{equation}
	\label{equ:SmoothingEffectKDE}
	\KDE_{\kappa}^{h}[\bsX] \xrightarrow{J \to \infty} \rho \ast \kappa^{h}.
\end{equation}

To the best of our knowledge, the consequences of the KDE smoothing effect on the asymptotic particle positions $\bsX(t)$ as $t \to \infty$ have not been fully explored, let alone utilized. 
To understand the implications of approximating $\rho_t$ using KDE, we begin by outlining some intuitive observations regarding the advantages and disadvantages of this approach:
\begin{itemize}
	\item[+] The density estimate and its gradient are computationally efficient and straightforward to calculate.
	\item[$-$] Due to the aforementioned smoothing effect, KDE may provide a poor approximation of $\rho_t$, particularly in high-dimensional settings.
	\item[+] In the long run ($t \to \infty$), the sample points arrange themselves such that $\hat{\den}_{t}^{h} \approx \trd$. 
	More precisely, the corresponding score $\nabla \log \hat{\den}_{t}^{h}$ interpolates the target score $\nabla \log \trd$ at the final points $\bsX(\infty)$, as for any $x \in \bR^{d}$,
	\begin{equation}
		\label{equ:score_interpolation}
		\hat v_t^{h}(x) = 0
		\quad \Longleftrightarrow \quad
		\nabla \log \hat{\den}_{t}^{h}(x)
		=
		\nabla \log \trd(x).
	\end{equation}
	In other words, the particles become stationary if and only if this interpolation property is satisfied.	
	Thus, we achieve an approximate representation of the target density $\trd$ as a mixture distribution, which is beneficial for various applications (see \Cref{section:Applications}).
	This bears some resemblance to score matching: The points $\bsX(\infty)$ are arranged in such a way that the score of their KDE matches the true score at the points themselves.
\end{itemize}

The smoothing effect \eqref{equ:SmoothingEffectKDE} and the approximation property \eqref{equ:score_interpolation} have two additional implications for the long-term behavior ($t \to \infty$):
\begin{itemize}
	\item[$\pm$] 
	The sample points are not distributed according to $\trd$, but rather according to a modified density $\check{\den}^{h}$, where $\check{\den}^{h} \ast \kappa^{h} \approx \trd$. 
	This results in the samples being more concentrated than they would be if drawn directly from $\trd$. 
	While this is disadvantageous for tasks requiring direct sampling from $\trd$, it has valuable applications in deconvolution problems that arise when using kernel mean embeddings, as discussed in \Cref{section:Outbedding}. 
	This approach also serves as an alternative to kernel herding \citep{chen2010super, lacoste2015sequential} and kernel conditional density operators \citep{schuster2020kcdo}.	
	\item[$+$] 
	The resulting points exhibit a more regular spatial distribution compared to independent samples (see \Cref{fig:SDEvsODE}).
	This regularity suggests the potential for super-root-$n$ convergence of Monte Carlo estimators with respect to $\check{\den}^{h}$ based on these points. 
	This phenomenon has been confirmed empirically (see \Cref{section:Convolution_Sampling}).
\end{itemize}

These observations suggest shifting the focus from accurately solving the ODE \eqref{equ:ODE} to using the approximate velocity field $\hat{\vel}_{t}^{h}$ to find (essentially deterministic) points $X_1(\infty), \dots, X_J(\infty)$, whose kernel density estimate $\KDE_{\kappa}^{h}[\bsX(\infty)]$ provides a good approximation of $\trd$. 
This approach has several practical applications, which are discussed in \Cref{section:Applications}.

It is important to emphasize that, although KDE typically performs poorly in high-dimensional settings, this limitation does not affect this new methodology, since the objective is no longer to estimate a density from given samples, but rather to generate points whose KDE provides an effective approximation of $\trd$.

Whether the particles driven by \eqref{equ:ODE_with_KDE} actually converge to fixed positions $X_1(\infty), \dots, X_J(\infty)$ remains an open question. 
In this paper, we assume this convergence. 
This assumption, along with the observation in \eqref{equ:score_interpolation}, motivates the following definition:

\begin{definition}
	Pairwise distinct points $\bsX = (X_1, \dots, X_J) \in (\bR^{d})^{J}$, $d,J \in \bN$, are called \emph{KDE points} with respect to the density $\trd \in C^{1}(\bR^{d}; \bR_{> 0})$, the kernel $\kappa \in C^{1}(\bR^{d}; \bR_{> 0})$, and the bandwidth $h > 0$, if $\hat{\vel}^{h}(X_j) = 0$ for each $j = 1, \dots, J$, where
	\begin{equation}
	\label{equ:FP_KDE_dynamics}
	\hat{\vel}^{h}
	\defeq
	\nabla \log \frac{\trd}{\KDE_{\kappa}^{h}[\bsX]}.	
	\end{equation}
\end{definition}

In addition to the observation in \eqref{equ:score_interpolation}, we now present a result showing that the procedure in \eqref{equ:ODE_with_KDE} reduces the Kullback--Leibler divergence $\dkl(\hat{\rho}_{t}^{h} \| \trd)$ between the kernel density estimate and the target at each time step, provided the bandwidth $h$ of the kernel $\kappa^{h}$ is sufficiently small.
\begin{theorem}
\label{theorem:Reduction_KL_KDE}
Let $d\in\bN$, $\trd,\kappa \in C^{\infty}(\bR^{d};\bR_{>0})$ be strictly positive probability density functions and, for $h>0$, denote $\kappa^{h}(x) \defeq h^{-d} \kappa(h^{-1}x)$, $x\in\bR^{d}$.
Further, let $X_{1},\dots,X_{J} \in \bR^{d}$, $J\in\bN$, be arbitrary points and consider the dynamics given by \eqref{equ:ODE_with_KDE} with $X_j(0) = X_{j}$.
Assume that
\begin{equation}
\label{equ:technical_decay_condition_for_prop}
\kappa(x) \, \Absval{ \log \frac{\hat{\den}_{0}^{h}(x)}{\trd(x)} } = o(\norm{x}^{1-d}),
\end{equation}
and that, for any $\delta,\varepsilon > 0$ and $y\in\bR^{d}$ there exist $h_{0}>0$ such that, for any $0<h\leq h_{0}$, the velocity field $\hat{v}_{0}^{h}$ satisfies
\begin{equation}
\label{equ:technical_decay_condition}
\int_{\bR^{d}\setminus \Ball{0}{\varepsilon}} \kappa^{h}(x) \, \norm{\hat{v}_{0}^{h}(y + x)}\, \mathrm{d} x
\leq
\delta.
\end{equation}
Then there exists $h_{\ast}>0$ such that, for every $0<h\leq h_{\ast}$,
\[
\tfrac{\mathrm{d}}{\mathrm{d}t}\big|_{t=0} \dkl(\hat{\rho}_{t}^{h} \| \trd)
\leq
0.
\]
\end{theorem}

\begin{proof}
The proof is provided in \Cref{appendix:Proofs}.
\end{proof}

\begin{remark}
	\label{remark:Remark_on_reduction_KL_KDE}
	\begin{itemize}
		\item
		The technical conditions \eqref{equ:technical_decay_condition_for_prop} and \eqref{equ:technical_decay_condition} imposes decay requirements on $\kappa$ and its gradient $\nabla \kappa$ in dependence of the target density $\trd$. 
		For instance, if $\nabla \log \trd \in L_{\kappa}^{1}$ and $\kappa(x) \, \absval{ \log \trd(x) } = o(\norm{x}^{1-d})$, the conditions are satisfied by commonly used kernels such as the Gaussian kernel.
		\item
		This result is not specific to the initial time point $t = 0$; the statement holds for any $t \geq 0$. 
		However, the value of $h_{\ast} > 0$ depends on the current positions $X_1, \dots, X_J$ and may vary over time. 
		A uniform value of $h_{\ast} > 0$ for all times may not exist, which is why, in the long term, \eqref{equ:score_interpolation} provides a better characterization of how the particles settle.
		\item
		The particles $X_j(t)$ are driven by the velocity field $\hat{\vel}_{t}^{h}$, which is the optimal velocity field for $\hat{\rho}_{t}^{h}$ in the sense of a gradient flow of the Kullback--Leibler divergence $\dkl(\hat{\rho}_{t}^{h} \| \trd)$ with respect to the 2-Wasserstein metric (see \Cref{section:FPE_as_gradient_flow}). 
		Practically, however, $\hat{\rho}_{t}^{h}$ itself cannot be evolved by $\hat{\vel}_{t}^{h}$ because we only transport the samples $X_j(t)$, whose distribution differs from $\hat{\rho}_{t}^{h}$. 
		Instead, $\hat{\den}_{t}^{h}$ evolves under a suboptimal velocity field $w_{t}^{h}$ as defined in \eqref{equ:perturbed_velocity_field}. 
		Nevertheless, we show in the proof that $\dkl(\hat{\den}_{t}^{h} \| \trd)$ is still reduced, as $\innerprod{w_{t}^{h}}{\hat{\vel}_{t}^{h}}_{L_{\hat{\rho}_{t}}^{2}} \geq 0$ for sufficiently small $h$.
	\end{itemize}
\end{remark}

One might object that the dependence of $h_{\ast}>0$ on the points $X_1, \dots, X_J$, and thus on $J$, renders the result less useful. 
For any given $J$ (and choice of points $X_1, \dots, X_J$), the value of $h_{\ast}>0$ may become too small, leading to an ``undersmoothed'' density $\hat\den_t^{h}$ with $J$ distinct peaks.
While increasing $J$ can mitigate this issue for a fixed bandwidth $h$, it may be necessary to decrease $h$ as $J$ increases, which complicates the situation.
To address this drawback, we present a second result that handles the limiting case of an infinite ensemble size, where $J \to \infty$ and the initial particles $X_1, \dots, X_J$ are replaced by an arbitrary initial density $\rho$:

\begin{theorem}
	\label{theorem:Reduction_KL_KDE_infinite_ensemble}
	Let $d\in\bN$, $\rho,\trd,\kappa \in C^{\infty}(\bR^{d};\bR_{>0})$ be strictly positive probability density functions and, for $h>0$, denote $\kappa^{h}(x) \defeq h^{-d} \kappa(h^{-1}x)$, $x\in\bR^{d}$.	
	Consider the dynamics given by
	\begin{equation}
		\label{equ:Infinite_ensemble_dynamics}
		\dot X(t)
		=
		\overline{\vel}_{t}^{h}(X(t)),
		\qquad
		X(0) \sim \rho,
		\qquad
		\overline{\vel}_{t}^{h}
		\defeq
		\nabla \log \frac{\den_{\fin}}{\overline{\den}_{t}^{h}},
		\qquad
		\overline{\den}_{t}^{h}
		\defeq
		\den_t \ast \kappa^{h},
	\end{equation}
	where $\rho_{t}$ denotes the probability density of $X(t)$.
	Let the following assumptions hold:
	\begin{enumerate}[label = (\roman*)]
	\item \label{item:compact_set}
	For any $\delta,\varepsilon > 0$ and $y\in\bR^{d}$ there exist $h_{1}>0$ such that, for any $0<h\leq h_{1}$, the velocity field $\overline{v}_{0}^{h}$ satisfies
	\begin{equation*}		
		\int_{\bR^{d}\setminus \Ball{0}{\varepsilon}} \kappa^{h}(x) \, \norm{\overline{v}_{0}^{h}(y + x)}\, \mathrm{d} x
		\leq
		\delta.
	\end{equation*}	
	\item \label{item:Cross_Correlation_in_L2}
	For any $\delta > 0$, there exist $h_{2}>0$ and a compact set $\cK \subseteq \bR^{d}$ such that, for any $0 \leq h \leq h_{2}$, the velocity field $\overline{\vel}_{0}^{h}$ satisfies
	\[
	\int_{\bR^{d}\setminus \cK} \rho(y) \, \norm{(\overline{\vel}_{0}^{h} \star \kappa^{h})(y)}^{2}\, \mathrm dy
	\leq
	\delta,
	\]
	where we set $\rho \ast \kappa^{0}\defeq \rho$ and $v\star \kappa^{0} \defeq v$
	(in particular, $\overline{\vel}_{0}^{h},\overline{\vel}_{0}^{h}\star \kappa^{h}\in L_{\rho}^{2}$).
	\item \label{item:L_1_condition}
	$\rho \hat{v}_{0}^{h}, \, \diver (\rho \hat{v}_{0}^{h}) \in L^{1}(\bR^{d})$ and $ \norm{ \big((\rho \overline{\vel}_{0}^{h})\ast \kappa\big)(x) } \absval{ \log \frac{\overline{\den}_{0}^{h}(x)}{\trd(x)} } = o(\norm{x}^{1-d}) $.
	\end{enumerate}	
	Then there exists $h_{\ast}>0$ such that, for any $0<h\leq h_{\ast}$, 
	\[
	\tfrac{\mathrm{d}}{\mathrm{d}t}\big|_{t=0} \dkl(\overline{\den}_{t}^{h} \| \trd)
	\leq
	0.
	\]
\end{theorem}

\begin{proof}
	The proof is provided in \Cref{appendix:Proofs}.
\end{proof}

\begin{remark}
Again, the technical conditions \ref{item:compact_set}---\ref{item:L_1_condition} can be ensured by imposing sufficient decay conditions on $\rho, \trd, \kappa$, and their gradients.
For instance, if the tails of $\rho, \trd, \kappa$ decay exponentially and the tails of $\log \trd$ have at most polynomial decay, the conditions are satisfied by commonly used kernels such as the Gaussian kernel.
\end{remark}

\subsection{Discussion of the interacting particle dynamics}
\label{section:discussion_interacting_particle_dynamics}

The velocity field
$\hat{\vel}^{h}
=
\nabla \log \frac{\trd}{\KDE_{\kappa}^{h}[\bsX]}$ in \eqref{equ:ODE_with_KDE} and \eqref{equ:FP_KDE_dynamics}, which arises from substituting the kernel density estimate for $\rho_{t}$ in $\vel_t^{\FP} = \nabla \log \frac{\trd}{\den_t}$, can be interpreted as an approximation to gradient descent on the objective function
\[
F(\bsX)
\coloneqq
\dkl ( \KDE_{\kappa}^{h}[\bsX] \, \|\,  \trd )
=
\int \KDE_{\kappa}^{h}(\bsX) \log \frac{\KDE_{\kappa}^{h}(\bsX)}{\trd},
\]
with respect to the particle positions. Specifically,
\begin{align*}
	-J\, \frac{\partial F}{\partial X_m}(\bsX)
	=
	\int \kappa^{h}(x - X_{m}) \nabla \log \frac{\trd}{\KDE_{\kappa}^{h}[\bsX]}(x) \, \mathrm dx
	\approx
	\nabla \log \frac{\trd}{\KDE_{\kappa}^{h}[\bsX]}(X_{m})
	=
	\hat{\vel}^{h}(X_{m}),
\end{align*}
where the approximation is accurate for sharply concentrated kernels (\( h \ll 1 \)) such as a centered Gaussian with small variance.
This dynamics bears a resemblance to Stein Variational Gradient Descent (SVGD) when the SVGD kernel is chosen as \( k(x,x') \coloneqq \kappa^{h}(x' - x) \):
\begin{align*}
\hat{\vel}^{h}(x)
&=
\nabla \log \trd (x) - \frac{J^{-1}\sum_{j=1}^{J} \nabla \kappa^{h}(x-x_{j})}{J^{-1}\sum_{j=1}^{J} \kappa^{h}(x-x_{j})},
\\
\hat{\vel}^{\SVGD}(x)
&=
J^{-1} \sum_{j=1}^{J} \kappa_{h}(x-x_{j}) \bigg( \nabla\log \trd (x_{j})
-
\frac{\nabla \kappa^{h}(x-x_{j})}{\kappa^{h}(x-x_{j})} \bigg).
\end{align*}
The corresponding terms operate in a similar manner to SVGD:
$\nabla \log \trd$ steers the particles toward the high-probability regions of the target distribution, while the term $J^{-1}\sum_{j=1}^{J} \nabla \kappa^{h}(x-x_{j})$ introduces a repulsive force between particles, preventing them from collapsing into the modes of $\trd$ \citep{liu2016stein}.

In light of the above interpretation as an approximate gradient descent, further modifications could be explored beyond optimizing the particle positions $\bsX$ alone. These include:
\begin{itemize}
	\item optimizing the bandwidth $h$;
	\item introducing individual bandwidths $h_{j} > 0$ or covariance matrices $h_{j} \in \bR^{d \times d}$ for each mixture component $X_{j}$, as in variable kernel density estimation \citep{scott2015multivariate,breiman1977variable,silverman1986density}, and optimizing over those;
	\item introducing weights $w_{j} \geq 0$ for each mixture component, $\KDE_{\kappa}^{h,\bsw}[\bsX] \defeq \sum_{j=1}^{J} w_{j} \, \kappa^{h}(x - X_j)$, and optimizing over those.
\end{itemize}
These potential enhancements represent directions for future work.

\subsection{Technical details: ODE solver, stopping criteria, bandwidth selection, and computational complexity}
\label{section:technical}

Through empirical testing, we found that implicit ODE solvers perform better than explicit ones (we used the \textsc{Matlab} solver \texttt{ode15s}), suggesting that the ODE \eqref{equ:ODE_with_KDE} is stiff. 
We define the termination condition based on the maximal velocity norm of the particles, $\nu_{t} \defeq \max_{j=1,\dots,J} \norm{\hat{\vel}_{t}^{h}(X_{j}(t))}$. 
When $\nu_{t}$ falls below a specified threshold $\varepsilon > 0$, we consider the particles to have become stationary and stop solving the ODE.

In all our examples, we use Gaussian kernels, and the bandwidth is chosen ad hoc, independent of the number of particles $J$. 
One could also consider decreasing the bandwidth as $J$ increases, following ideas similar to those in \citet{wang2019accelerated} and \citet{liu2019understanding} for selecting an optimal bandwidth.

Assuming that the number of right-hand side evaluations required to solve the ODE \eqref{equ:ODE_with_KDE} does not scale with the number of particles $J$, the computational complexity of constructing $J$ KDE points is $\cO(J^{2})$. 
However, verifying this assumption in practice is challenging.

%% file: Section_Applications.tex

\section{Applications}
\label{section:Applications}

In the previous section we changed our course from approximately solving the Fokker--Planck equation to the estimation of the target density $\trd$ by a mixture distribution, in contrast to \citet{wang2019accelerated} and \citet{liu2019understanding}, who employ the dynamics \eqref{equ:ODE_with_KDE} as a particle flow method.
Note that this dynamics does not require the normalizing constant of $\trd$, since $\hat v_t^{h}$ relies on $\trd$ only through the gradient of its logarithm.

In this section, we discuss several applications based on this methodology. These applications can be grouped into two distinct areas, reflecting the observations in \ref{item:KDE_points_KDE_property} and \ref{item:KDE_points_deconvolution_property}, introduced earlier and summarized in \Cref{table:application_areas}. The first area leverages the kernel density estimate $\KDE_{\kappa}^{h}[\bsX (\infty)]$ as an approximation of $\trd$, while the second area utilizes the KDE points $\bsX (\infty)$ directly.

\begin{table}[t!]
	\centering
	\caption{%
		The KDE points are applied in two main areas: one uses the kernel density estimate $\KDE_{\kappa}^{h}[\bsX (\infty)]$ as an approximation of $\trd$, and the other involves employing the KDE points $\bsX (\infty)$ as samples from $\check{\den}^{h}$.
	}
	\label{table:application_areas}
	\bgroup
	\begin{tabular}{m{9em}<{\centering} | m{9em}<{\centering} || m{9em}<{\centering} | m{9em}<{\centering} } 
		\toprule
		\multicolumn{2}{c||}{\cellcolor{ourblue!15}\textbf{Application Area A}, \Cref{section:VI}} &
		\multicolumn{2}{c}{\cellcolor{ouryellow!30}\textbf{Application Area B}, \Cref{section:Convolution_Sampling}} \\
		\midrule
		\midrule
		\multicolumn{2}{m{19.12em}<{\centering}||}{\cellcolor{ourblue!8} Using the estimate $\KDE_{\kappa}^{h}[\bsX (\infty)] \approx \trd$:
			\newline
			Variational Bayesian methods} &
		\multicolumn{2}{m{19.12em}<{\centering}}{\cellcolor{ouryellow!16} Using the KDE points $\bsX(\infty)$ as (super-root-$n$) samples from $\check{\den}^{h}$} \\
		\midrule
		\cellcolor{ourblue!5} independent or stratified samples from
		$\KDE_{\kappa}^{h}[\bsX (\infty)]$,
		with importance reweighting,
		\Cref{section:Direct_Sampling}
		& 
		\cellcolor{ourblue!5} Transporting QMC points (and others) to $\KDE_{\kappa}^{h}[\bsX (\infty)]$, with importance reweighting,
		\Cref{section:Transport_Sampling}
		& 
		\cellcolor{ouryellow!10}
		Outbedding:
		Inversion of kernel mean embeddings,
		\Cref{section:Outbedding}
		& 
		\cellcolor{ouryellow!10} Resampling within		
		sequential Monte
		\hspace{2em}
		Carlo,
		\hspace{2em}		
		\Cref{section:SMC_resampling}
		\\
		\bottomrule
	\end{tabular}
	\egroup
\end{table}

%% file: Section_VI.tex
\subsection{\hlc[ourblue!8]{Leveraging the KDE: Variational Bayesian methods}}
\label{section:VI}

Variational inference \citep{jordan1999introduction,wainwright2008graphical,attias1999variational,winn2005variational} aims to approximate a target (posterior) density $\trd$ by a parametrized density $q_{\phi}$, where $\phi \in \Phi$ represents the (family of) parameters. 
Typically, the distance between $\trd$ and $q_{\phi}$ is measured using the Kullback--Leibler divergence. 
The objective is to solve the following minimization problem:
\begin{equation}
	\label{equ:VI_formulation}
	\phi^{\ast}
	=
	\argmin_{\phi \in \Phi} \dkl(q_{\phi} \| \trd),
\end{equation}
although, as \citet{blei2017variational} note, ``any procedure which uses optimization to approximate a density can be termed variational inference.''

In our case, the target density is approximated by a kernel density estimate $\KDE_{\kappa}^{h}[\bsX]$, i.e., by a mixture distribution where the parameters $\phi = \bsX \in (\bR^{d})^{J}$ correspond to the particle positions. 
When the dimension $d$ or the number of particles $J$ is large, the resulting variational inference problem becomes particularly challenging due to the high number of parameters ($d \times J$) and the non-convexity of the objective function.

\Cref{theorem:Reduction_KL_KDE} demonstrates that, for sufficiently small bandwidth $0 < h \leq h_{\ast}$ of the kernel $\kappa^{h}$, the ODE \eqref{equ:ODE_with_KDE} can indeed be viewed as a minimization procedure of the form \eqref{equ:VI_formulation}. 
However, it is important to emphasize that a bound $h_{\ast}$ that works uniformly for all times may not exist (see \Cref{remark:Remark_on_reduction_KL_KDE}). 
Moreover, in the long term and for large ensemble sizes $J$, the particles tend to arrange themselves according to \eqref{equ:score_interpolation}, as discussed in \Cref{section:Approach}.

%% file: Section_Direct_Sampling.tex

\subsubsection{\hlc[ourblue!8]{Direct sampling}}
\label{section:Direct_Sampling}

A key task in Bayesian inference is the computation of expected values $\bE_{\trd}[f] = \int_{\bR^{d}} f \, \trd$ for certain quantities of interest $f \in L_{\trd}^{1}$. 
Once we have established an approximation $\hat{\den}_{t}^{h} = \KDE_{\kappa}^{h}[\bsX (t)] \approx \trd$ for sufficiently large $t > 0$ using the dynamics in \eqref{equ:ODE_with_KDE}, and assuming that independent samples from the kernel $\kappa$ are easy to generate (e.g., if $\kappa$ is Gaussian), we can efficiently sample from $\hat{\den}_{t}^{h}$ using the \emph{composition method} \citep[Section 2.3.3]{rubinstein2016simulation}:
\begin{algorithm}[Direct sampling from $\hat{\den}_{t}^{h}$ via the composition method]
	\label{alg:direct_sampling_of_mixtures}
	\
	\begin{enumerate}
		\item
		Draw $Z \sim \kappa$;
		\item
		Draw $\nu$ from the uniform distribution $\bP_{\Unif}$ on $\Omega = \{1, \dots, J\}$, i.e., $\bP_{\Unif}(\{j\}) = J^{-1}$;
		\item
		Set $Y = hZ + X_{\nu}(t)$.
	\end{enumerate}
\end{algorithm}

Using this method, a large number $K \in \bN$ of independent samples $Y_k \stackrel{\iid}{\sim} \hat{\den}_{t}^{h}$, $k = 1, \dots, K$, can be generated with very low computational cost. 
If there are concerns about the approximation quality of $\hat{\den}_{t}^{h}$, the importance sampling trick (\citealt[Section 5.7]{rubinstein2016simulation}; \citealt[Section 3.3]{robert2004monte}) can be employed to account for the approximation error:
\begin{equation}
	\label{equ:importance_sampling}
	\bE_{\trd}[f]
	=
	\int_{\bR^{d}} f \, \frac{\trd}{\hat{\den}_{t}^{h}} \, \hat{\den}_{t}^{h}
	\approx
	\frac{1}{K} \sum_{k=1}^{K} f(Y_{k}) \, \frac{\trd(Y_{k})}{\hat{\den}_{t}^{h}(Y_{k})}.
\end{equation}

If the target density $\trd$ is known only up to a normalizing constant, i.e., we can evaluate $\tilde{\rho}_{\fin} = Z \trd$ with an unknown constant $Z = \int \tilde{\rho}_{\fin}$, then self-normalized importance sampling\footnote{Note that, unlike the importance sampling estimator \eqref{equ:importance_sampling}, self-normalized importance sampling typically yields a \emph{biased} estimator.} \citep[Section 5.7.1]{rubinstein2016simulation} may be used:
\begin{equation}
	\label{equ:self_normalized_importance_sampling}
	\bE_{\trd}[f]
	=
	\frac{\int_{\bR^{d}} f \, \tilde{\rho}_{\fin}}{\int_{\bR^{d}} \tilde{\rho}_{\fin}}
	=
	\frac{\int_{\bR^{d}} f \, \frac{\tilde{\rho}_{\fin}}{\hat{\den}_{t}^{h}} \, \hat{\den}_{t}^{h}}{\int_{\bR^{d}} \frac{\tilde{\rho}_{\fin}}{\hat{\den}_{t}^{h}} \, \hat{\den}_{t}^{h}}
	\approx
	\frac{\sum_{k=1}^{K} f(Y_{k}) \, \frac{\tilde{\rho}_{\fin}(Y_{k})}{\hat{\den}_{t}^{h}(Y_{k})} }{\sum_{k=1}^{K} \frac{\tilde{\rho}_{\fin}(Y_{k})}{\hat{\den}_{t}^{h}(Y_{k})}}.
\end{equation}

As a side note, one might aim for an importance sampling estimator with minimal variance \citep[Theorem 3.12]{robert2004monte}, applying the dynamics \eqref{equ:ODE_with_KDE} to the modified target density $\tilde{\rho}_{\fin} \propto \absval{f} \trd$. 
However, this approach will not be pursued here.

Since $\hat{\den}_{t}^{h}$ is a mixture distribution, \emph{stratified} sampling \citep[Chapter 5.5]{rubinstein2016simulation} can be employed as an alternative to the independent sampling in \Cref{alg:direct_sampling_of_mixtures}. 
In this approach, exactly $L \in \bN$ independent samples are drawn from each of the $J$ mixture components, resulting in a total of $K = JL$ stratified samples.
The corresponding Monte Carlo estimator can be shown to have a lower variance \citep[Proposition 5.5.1]{rubinstein2016simulation}.

In summary, the dynamics \eqref{equ:ODE_with_KDE} can be leveraged to obtain a good approximation $\hat{\rho}_{t}^{h}$ to the target density $\trd$. 
Since it is straightforward and computationally inexpensive to generate either independent or stratified samples from $\hat{\rho}_{t}^{h}$, Monte Carlo methods are a natural choice, with (self-normalized) importance sampling providing a way to correct for discrepancies in the approximation.
This approach offers a promising alternative to MCMC methods \citep{hastings1970monte, robert2004monte}, which are often slow due to correlations between successive samples.

%% file: Section_Transport_Sampling.tex

\subsubsection{\hlc[ourblue!8]{Sampling via transport maps: super-root-$n$ convergence}}
\label{section:Transport_Sampling}

In the previous subsection, we discussed how to obtain \emph{independent} (or stratified) samples from $\hat{\den}_{t}^{h}$, which represents a significant improvement over MCMC methods, where convergence is often slow due to dependence between the samples.
However, since this approach remains a Monte Carlo approximation, a convergence rate better than $K^{-1/2}$ cannot be expected.

However, because $\hat{\den}_{t}^{h}$ is a mixture distribution, it is possible to generate higher-order point sequences from $\hat{\den}_{t}^{h}$ that achieve a super-root-$n$ (or, more accurately, super-root-$K$) convergence rate \citep{Cui2023quasimonte,klebanov2023transporting}.
The key idea is to construct an \emph{exact} transport map $T$ (up to the solution of ODEs with explicit right-hand side) from $\kappa^{h}$ to $\hat{\den}_{t}^{h}$, i.e., $T_{\#}\kappa^{h} = \hat{\den}_{t}^{h}$, and apply this map to $K$ quasi-Monte Carlo (QMC) points \citep{niederreiter1992random,fang1993number,caflisch1998monte,dick2013high} corresponding to the density $\kappa^{h}$.
For a Gaussian kernel $\kappa^{h}$, such QMC points can be obtained through an analytical transformation \citep{kuo2010lattice,klebanov2023transporting}.\footnote{Alternatively, sparse grids \citep{smolyak1963quadrature,zenger1991sparse,gerstner1998numerical} or higher-order nets \citep{dick2010digital} can be used.
For simplicity and proof of concept, we focus on QMC (specifically, the $d$-dimensional Halton sequence is employed in all our computations).}
As before, (self-normalized) importance sampling can be employed to correct for any approximation error between $\hat{\den}_{t}^{h}$ and $\trd$.

Hence, the overall procedure consists of solving two systems of ODEs: one for constructing the $J$ KDE points and the corresponding density $\hat{\den}_{t}^{h}$, and the other for transporting $K$ QMC points to $\hat{\den}_{t}^{h}$, which we refer to as \emph{KDE-QMC points}. 
\Cref{fig:Rezende_KDE_QMC_vs_MCMC} illustrates this methodology using a bimodal target density in $d=2$ dimensions from \citep{rezende2015variational}, defined as
\[
\trd
\, \propto\, 
\exp(-U),
\qquad
U(x)
=
\tfrac{1}{2}\Absval{\frac{\norm{x}-2}{0.4}}^{2} - 
\log\left( 
e^{-\tfrac{1}{2}\Absval{\frac{x_{1}-2}{0.6}}^{2}} + 
e^{-\tfrac{1}{2}\Absval{\frac{x_{2}+2}{0.6}}^{2}}
\right).
\]
To ensure a fair comparison, we break the symmetry by initializing the density as $\simpd = \cN\big( (\tfrac{1}{2},\tfrac{1}{2}), \Id_{2} \big)$.
For small to moderately large numbers $J$ of KDE points, this results in more points and higher density values of $\hat{\den}_{t}^{h}$ near the right mode compared to the left, as seen in \Cref{fig:Rezende_KDE_QMC_vs_MCMC} (left and middle).
This imbalance is addressed by importance reweighting, as shown in \Cref{fig:Rezende_KDE_QMC_vs_MCMC} (middle), where the importance weights on the left are larger than those on the right.
If importance sampling alone cannot fully correct these imbalances (e.g., when regions of high density lack KDE points even for large $J$), a strategy similar to simulated annealing can be applied, as discussed in \Cref{section:Annealing}.

\Cref{fig:Rezende_KDE_QMC_vs_MCMC} (right) compares the performance of KDE-QMC points against MCMC (random walk Metropolis--Hastings algorithm with an optimally tuned acceptance rate), as well as independent and stratified samples from $\hat{\den}_{t}^{h}$, for estimating the expected value $\bE_{\trd}[f]$ with $f(x) = x$. 
The error is measured using the Euclidean norm and plotted as a function of the total number $K$ of samples.
The number of KDE points used in $\hat{\den}_{t}^{h}$ (and consequently for KDE-QMC points, independent, and stratified samples) is set to $J = \big\lceil K^{1/2} \big\rceil$, ensuring that the computational complexity for obtaining the samples is $\cO(K)$ across all methods (cf.\ \Cref{section:technical}).

\begin{figure}[t]
	\centering
	\begin{subfigure}[b]{0.32\textwidth}
		\centering
		\includegraphics[width=\textwidth]{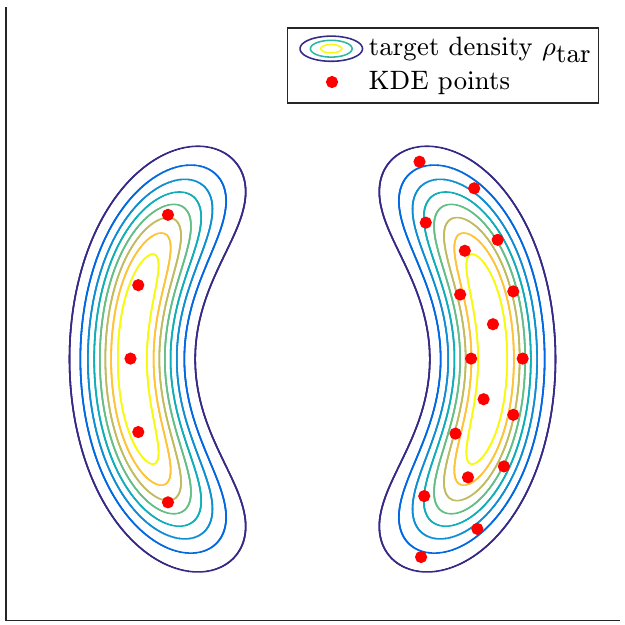}	
	\end{subfigure}
	\hfill
	\begin{subfigure}[b]{0.32\textwidth}
		\centering
		\includegraphics[width=\textwidth]{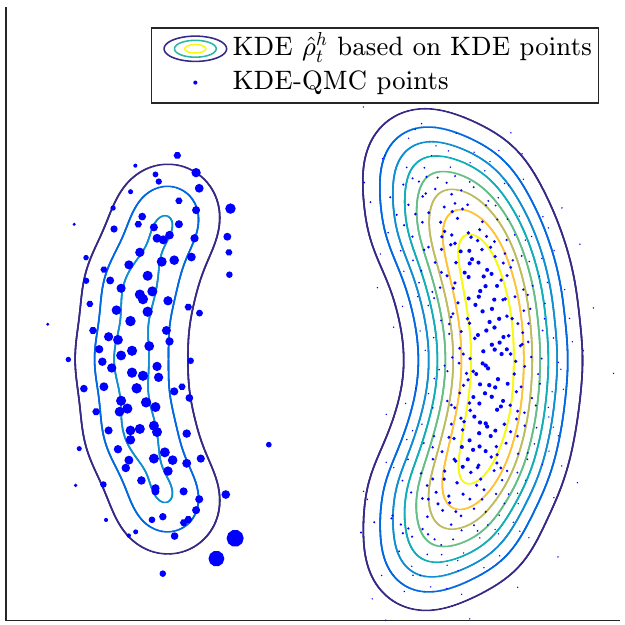}	
	\end{subfigure}
	\hfill
	\begin{subfigure}[b]{0.32\textwidth}
		\centering
		\includegraphics[width=\textwidth]{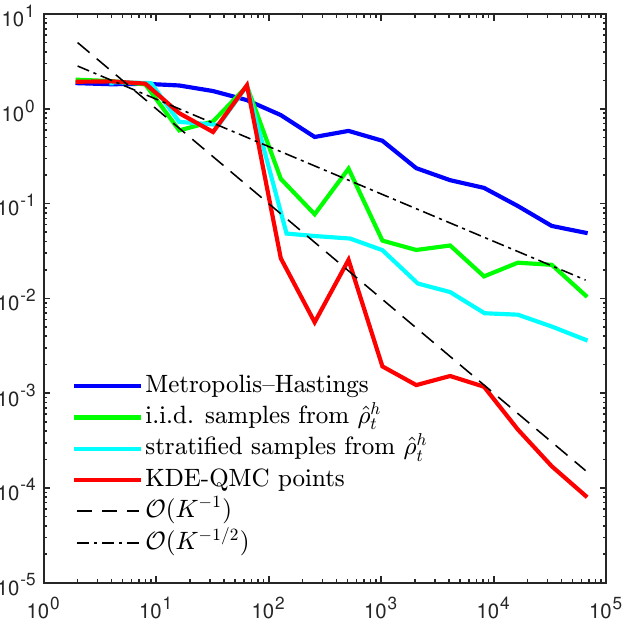}	
	\end{subfigure}
	\caption{
		\emph{Left:} The target density $\trd$ with $J = 23$ KDE points. 
		Due to the asymmetric initial density $\simpd$, significantly more KDE points are concentrated near the right mode than the left (see \Cref{section:Annealing} for potential corrections). 
		\emph{Middle:} The estimated density $\hat{\den}_{t}^{h}$ with $K = 526$ KDE-QMC points. 
		The asymmetry is corrected through larger importance weights (indicated by marker sizes) near the left mode. 
		\emph{Right:} Error estimates for $\bE_{\trd}[f]$ with $f(x) = x$ using MCMC, independent and stratified samples from $\hat{\den}_{t}^{h}$, and KDE-QMC points. 
		As expected, the first three methods exhibit a convergence rate of $K^{-1/2}$ (the error plots display the mean over ten independent runs), whereas the KDE-QMC points achieve a convergence rate of approximately $K^{-1}$.
	}
	\label{fig:Rezende_KDE_QMC_vs_MCMC}
\end{figure}

%% file: Section_Convolution_Sampling.tex

\subsection{\hlc[ouryellow!16]{Employing the KDE points: Sampling from $\check{\den}^{h}$ given only $\trd = \check{\den}^{h} \ast \kappa^{h}$}}
\label{section:Convolution_Sampling}

As outlined in \Cref{section:Approach}, the dynamics defined by \eqref{equ:ODE_with_KDE} generates points $z_{j} \defeq X_{j}(\infty)$ that are not distributed according to $\trd$, but rather according to $\check{\den}^{h}$, where $\check{\den}^{h}$ is a probability density satisfying $\check{\den}^{h} \ast \kappa^{h} \approx \trd$ in the sense of \eqref{equ:score_interpolation}.
While this property of KDE points may be undesirable for most classical applications—where samples from $\trd$ are typically needed—there exist specific applications that benefit from samples drawn from $\check{\den}^{h}$.
One such application is discussed in \Cref{section:Outbedding}.

Therefore, we now turn to the scenario where we seek to sample from $\check{\den}^{h}$, but can only evaluate $\trd = \check{\den}^{h} \ast \kappa^{h}$ and its gradient.
This setup parallels the frameworks of kernel herding \citep{chen2010super,Bach2012HerdingConditionalGradient,lacoste2015sequential}, Bayesian Monte Carlo \citep{ghahramani2003bayesian}, and Sequential Bayesian Quadrature \citep{huszar2012herding}.

To numerically verify the claims made in \Cref{section:Approach}—namely, that the dynamics \eqref{equ:ODE_with_KDE} generates points that are $\check{\den}^{h}$-distributed and exhibit super-root-$n$ convergence—we generate KDE points from the Gaussian mixture density $\check{\den}^{h}$ analyzed by \citet{chen2010super} and \citet{huszar2012herding}. 
We then compare these KDE points to independent samples from $\check{\den}^{h}$, as well as to those produced by kernel herding and sequential Bayesian quadrature (SBQ).
The density, along with 78 KDE points, and the first 78 herding and SBQ samples, is visualized in \Cref{fig:Chen_Density_Herding_KDE}.
For a fair comparison, both the target density and the kernel $\kappa^{h}$ are exactly as specified in the works of \citet{chen2010super} and \citet{huszar2012herding}.\footnote{We make use of the code provided by \citet{huszar2012herding} (available at \url{github.com/duvenaud/herding-paper}), and we are grateful to the authors for sharing it.}

\begin{figure}[t]
	\centering
	\begin{subfigure}[b]{0.49\textwidth}
		\centering
		\includegraphics[width=\textwidth]{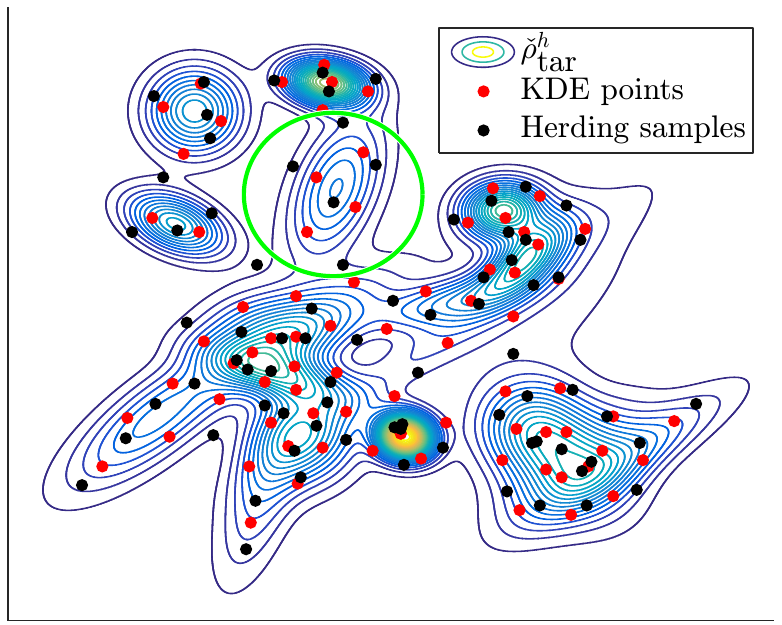}		
	\end{subfigure}
	\hfill
	\begin{subfigure}[b]{0.49\textwidth}
		\centering
		\includegraphics[width=\textwidth]{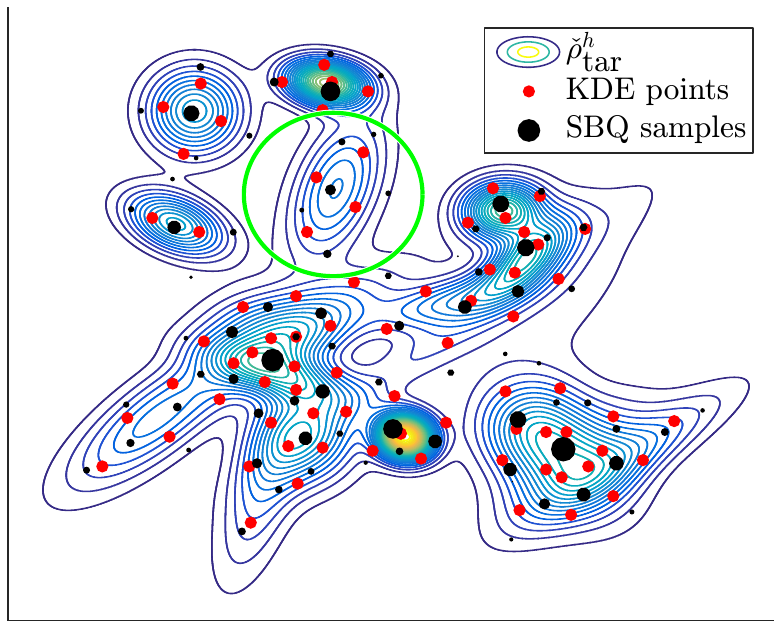}		
	\end{subfigure}
	\caption{
		Comparison of 78 KDE points with the first 78 samples from kernel herding (left) and sequential Bayesian quadrature (right; marker sizes correspond to sample weights).
		While all three point sets are fairly evenly distributed, herding and SBQ samples are more frequently placed in regions of lower density.
		This behavior is due to the sequential nature of their generation, as illustrated by the green circles: 
		Once a ``central'' point is fixed at the optimal position for a given time step, it cannot be adjusted later, forcing subsequent samples into positions further away.
		In contrast, KDE points benefit from greater flexibility since they are generated simultaneously.
		In addition, KDE points do not require solving non-convex optimization problems, as discussed in \Cref{remark:Oprimization_issues_SBQ}.
	}
	\label{fig:Chen_Density_Herding_KDE}
\end{figure}

The sequential nature of kernel herding and SBQ does not permit an existing point to be adjusted once further samples are placed nearby.
As a result, later samples often end up in regions of relatively low density.
This effect is illustrated by the green circles in \Cref{fig:Chen_Density_Herding_KDE}, where KDE points do not exhibit this behavior, since they are generated collectively through the dynamics \eqref{equ:ODE_with_KDE}.
However, sequential methods have the advantage of allowing samples to be added ``ad hoc'' as needed, whereas increasing the number of KDE points requires rerunning the entire dynamics \eqref{equ:ODE_with_KDE}.

\Cref{fig:ConvergenceQuadratureError} shows the convergence rate of quadrature rules for several point sets $\bz^{(J)} = (z_{1}^{(J)},\dots,z_{J}^{(J)})$ with corresponding weights $\bw^{(J)} = (w_{1}^{(J)},\dots,w_{J}^{(J)})$ as $J\to\infty$.
The left plot presents the average quadrature error,
\[
\cE_{\bz^{(J)},\bw^{(J)} }[f]
\defeq
\Absval{\bE_{\check{\den}^{h}}[f] - \sum_{j=1}^{J} w_{j}^{(J)} f(z_{j}^{(J)})},
\]
calculated over 50 randomly selected test functions $f$ from the unit ball in the reproducing kernel Hilbert space $\cH_{k}$ corresponding to the kernel $k(x,x') \coloneqq \kappa^{h}(x - x')$. 
The kernel used is Gaussian, ensuring it is symmetric and positive definite.
For details on the choice of test functions, see \cite{huszar2012herding}.
The right plot shows the maximum mean discrepancy (MMD) between $\check{\den}^{h}$ and the discrete measure $\bP_{\bz^{(J)},\bw^{(J)}} \coloneqq \sum_{j=1}^{J} w_{j}^{(J)} \delta_{z_{j}^{(J)}}$,
\[
\MMD(\check{\den}^{h} , \bP_{\bz^{(J)},\bw^{(J)}})
=
\sup_{\norm{f}_{\cH_{k}} \leq 1}
\cE_{\bz^{(J)},\bw^{(J)}}[f]
=
\norm{\mu_{\check{\den}^{h}} - \mu_{\bP_{\bz^{(J)},\bw^{(J)}}}}_{\cH_{k}},
\]
where $\mu_{\bP} \coloneqq \int k(x,\quark)\, \mathrm{d}\bP(x) \in \cH_{k}$ denotes the kernel mean embedding of the distribution $\bP$ \citep{berlinet2004rkhs,muandet2017kernel}.

\begin{figure}[t]
	\centering
	\begin{subfigure}[b]{0.49\textwidth}
		\centering
		\includegraphics[width=\textwidth]{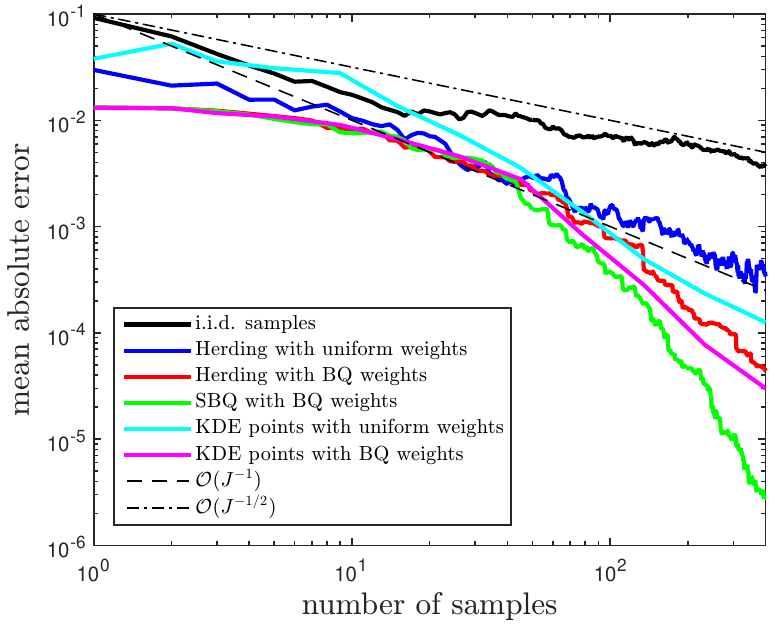}
		\caption{dynamics via the SDE \eqref{equ:SDE}}			
	\end{subfigure}
	\hfill
	\begin{subfigure}[b]{0.49\textwidth}
		\centering
		\includegraphics[width=\textwidth]{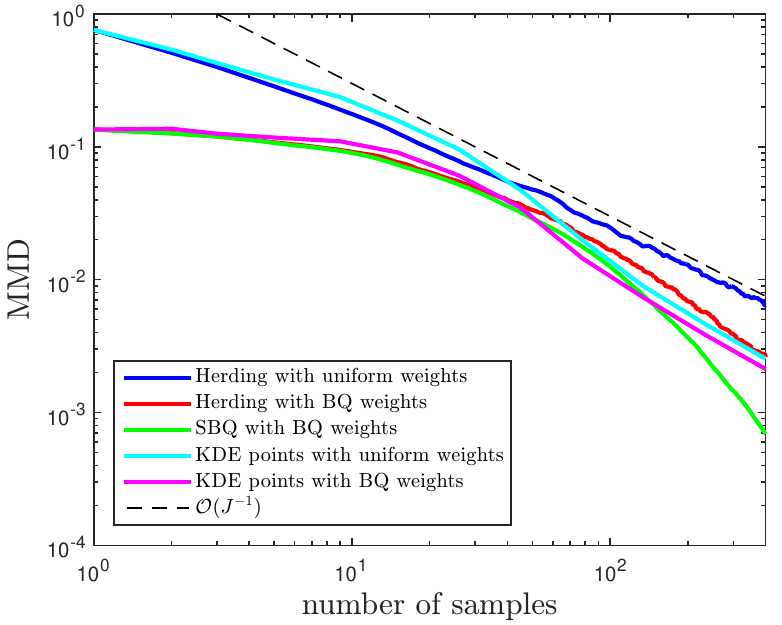}
		\caption{dynamics via ODE \eqref{equ:FokkerAsContinuity}}			
	\end{subfigure}
	\caption{Performance comparison of KDE points, kernel herding, and SBQ using both uniform and SBQ weights, measured by the average quadrature error over 50 randomly selected functions (left) and by maximum mean discrepancy (MMD) (right).}
	\label{fig:ConvergenceQuadratureError}
\end{figure}

The point sets compared in the plots include KDE points with uniform weights, kernel herding with uniform weights, and SBQ with SBQ weights.
Additionally, following the suggestion by \cite{huszar2012herding}, we apply SBQ weights to both KDE points and herding samples.

As shown in \Cref{fig:ConvergenceQuadratureError}, KDE points slightly outperform kernel herding, though the convergence rates appear to be similar.
Both methods benefit slightly from using SBQ weights instead of uniform weights.
SBQ, however, appears to outperform the other methods, though its convergence rate remains unknown \citep{huszar2012herding}, and it comes with a computational complexity of $\cO(J^{3})$ for $J$ points.

\begin{remark}
	\label{remark:Oprimization_issues_SBQ}
	Notably, the sequential procedures of kernel herding and SBQ require solving a non-convex optimization problem at each step to determine the next point.
	In practice, this is often approached by optimizing over either
	\begin{itemize}
		\item
		a large number of randomly chosen samples within a predefined cube,
		\vspace{-3ex}
		\begin{quotation}
			\noindent
			{\small 
				\begin{verbatim}	
					% This demo searches for possible next locations by drawing from a uniform prior
					% with the following range.  This is a bad idea in high dimensions,
					% and is only done here so that all the code will be really simple.
					range = [ -6, 6; -5 3];
				\end{verbatim}
			}
		\end{quotation}
		(code by \citet{huszar2012herding} on \url{github.com/duvenaud/herding-paper});
		\item
		or a large number of initially given random samples from $\check{\den}^{h}$, performing a procedure known as ``thinning'' or ``quantization'',
		\begin{quotation}
			\noindent
			{\small 
				In general, [finding the next quadrature point] will yield a non-convex optimization problem, and thus cannot be solved with guarantees, even with gradient descent.
				[...] we approach [this step] by performing an exhaustive search over $M$ random samples from $p$ [...]. 
				We follow the idea from the kernel herding paper \citep{chen2010super} to choose the best $N$ ``super-samples'' out of a large set of $M$ samples.
			}
		\end{quotation}
	\citep{lacoste2015sequential}.
	\end{itemize}
	Both approaches simplify the problem by bypassing the more complex optimization steps.
	In contrast, our construction of KDE points does not involve solving non-convex optimization problems, thereby avoiding the complexities and workarounds inherent in such approaches. 
	Although KDE points require solving an ODE, which entails multiple evaluations of both the target and its gradient, this process is deterministic and can be efficiently handled using modern numerical solvers.
	While direct comparisons between the computational costs of these methods are not straightforward, the absence of non-convex optimization steps in the KDE approach provides clear procedural simplicity.
\end{remark}

%% file: Section_Outbedding.tex

\subsubsection{\hlc[ouryellow!16]{Kernel Mean Outbedding: Inverting a Kernel Mean Embedding}}
\label{section:Outbedding}

A specific instance of the scenario described in \Cref{section:Convolution_Sampling}, where we aim to sample from $\check{\den}^{h}$ and can only evaluate $\trd = \check{\den}^{h} \ast \kappa^{h}$ and its gradient, arises in the context of inverting a kernel mean embedding. This inversion is necessary, for instance, in the final step of conditional mean embedding methods.
Kernel methods, grounded in the mathematical framework of reproducing kernel Hilbert spaces (RKHS), are fundamental tools in machine learning and statistics \citep{schoelkopf2018learning}. 
The central idea is to embed points $x, x' \in \cX$ (representing data, observations, parameters, etc.) into an RKHS $\cH_{k}$, associated with a symmetric positive definite kernel $k\colon \cX \times \cX \to \bR$. 
This is done via the canonical feature map $\varphi \colon \cX \to \cH_{k}$, where $\varphi(x) \defeq k(x, \quark)$. 
In this space, many problems become linear, and the inner product $\innerprod{\varphi(x)}{\varphi(x')}_{\cH_{k}} = k(x,x')$ can be computed efficiently using the ``kernel trick" which avoids the need to explicitly evaluate the feature map or the inner product in $\cH_{k}$ \citep{berlinet2004rkhs, steinwart2008support}.

Apart from the embedding of points $x\in \cX$, the embedding of probability distributions $\bP$ or densities $\rho$ on $\cX$ into the same RKHS $\cH_{k}$ have recently gained a lot of attention and popularity.
For an overview on methods based on such \emph{kernel mean embeddings} (KMEs; \citealt{smola2007embedding,berlinet2004rkhs}),
\[
\mu_{\bP}
\defeq
\int_{\cX} k(x,\quark) \, \mathrm{d}\bP(x) \in \cH_{k},
\qquad
\mu_{\rho}
\defeq
\int_{\cX} k(x,\quark) \, \rho(x)\, \mathrm{d}x \in \cH_{k},
\]
see \citet{muandet2017kernel}.
In particular, the \emph{conditional mean embedding} (CME; \citealt{song2009hilbert,fukumizu2013kernel,klebanov2020rigorous}) performs the conditioning of random variables $X$ in $\cX$ and $Y$ in $\cY$ by a linear-algebraic transformation in the corresponding RKHSs $\cH_{k}$, $\cH_{\ell}$, $\ell$ being a kernel on $\cY$ with canonical feature map $\psi(y) = \ell(y,\quark)$:
\begin{equation}
	\label{equ:myCME}
	\mu_{\bP_{Y|X = x}}
	=
	\mu_{\bP_{Y}} + (C_{X}^{\dagger} C_{XY})^{\ast} \, (\varphi(x) - \mu_{\bP_{X}})
	\qquad
	\text{for $\bP_{X}$-a.e.\ $x\in\cX$.}
\end{equation}
Here, $\mu_{\bP_{Y|X = x}} \in \cH_{\ell}$ is the embedding of the conditional distribution $\bP_{Y|X=x}$ while $C_{X}$ and $C_{XY}$ denote the (cross-) covariance operators of $\varphi(X)$ and $\psi(Y)$, respectively (for details see \citealt{klebanov2020rigorous,klebanov2020linear}).

One crucial step of these methodologies is the inversion of the kernel mean embedding, which we term \emph{kernel mean outbedding}, i.e.\ the recovery of $\bP_{X}$ from its embedding $\mu_{X}$ (e.g., in the case of CMEs, it is crucial to regain the conditional distribution of interest $\bP_{Y|X=x}$ from its embedding  $\mu_{\bP_{Y|X = x}}$).
While the KME is known to be injective as a function from
\[
\cP_{k}
\defeq
\{ \bP \mid \bP \text{ is a probability\ measure\ on } \cX \text{ with } \int_{\cX} \norm{ \varphi(x) }_{\cH_{k}} \, \rd \bP (x) < \infty \}
\]
into $\cH_{k}$ for a large class of kernels $k$, so-called \emph{characteristic}
kernels, the outbedding step is highly non-trivial, especially in high dimensions.
Here, we focus on $\cX = \bR^{d}$ as well as on characteristic and translation-invariant kernels, i.e.\ $k(x,x') = \kappa^{h}(x-x')$ for some symmetric and positive definite probability density $\kappa^{h} \colon \cX \to \bR$, and, as usual, work with probability densities $\rho$ rather than distributions $\bP$.
In this case, $\mu_{X} = \rho \ast \kappa^{h}$ and, thereby, outbedding corresponds to a \emph{deconvolution}.
Note that, since $\norm{ \varphi(x) }_{\cH_{k}} = \sqrt{k(x,x)} = \sqrt{\kappa^{h}(0)}$, every probability measure on $\cX$ lies in the space $\cP_{k}$ and, consequently, has a well-defined KME.

Now, if we set $\trd \defeq \mu_{\den}$, we get $\check{\den}^{h} = \den$ because $k$ is characteristic. This makes it one of the rare cases from \Cref{section:Convolution_Sampling} where $\trd$ can be evaluated directly while sampling from the unknown density $\check{\den}^{h}$.
In such cases, what might seem like a disadvantage of KDE turns out to be beneficial.
Note also that, in this specific case, the ``deconvolved'' density $\check{\den}^{h}$ with $\check{\den}^{h} \ast \kappa^{h} = \trd$ is guaranteed to exist.
Hence, we can apply the strategy from \Cref{section:Convolution_Sampling} to obtain (super-) samples from $\check{\den}^{h}$ by only evaluating $\trd$ and its gradient.

\subsubsection{\hlc[ouryellow!16]{Resampling within Sequential Monte Carlo}}
\label{section:SMC_resampling}

\begin{figure}[t]
	\centering
	\begin{subfigure}[b]{0.32\textwidth}
		\centering
		\includegraphics[width=\textwidth]{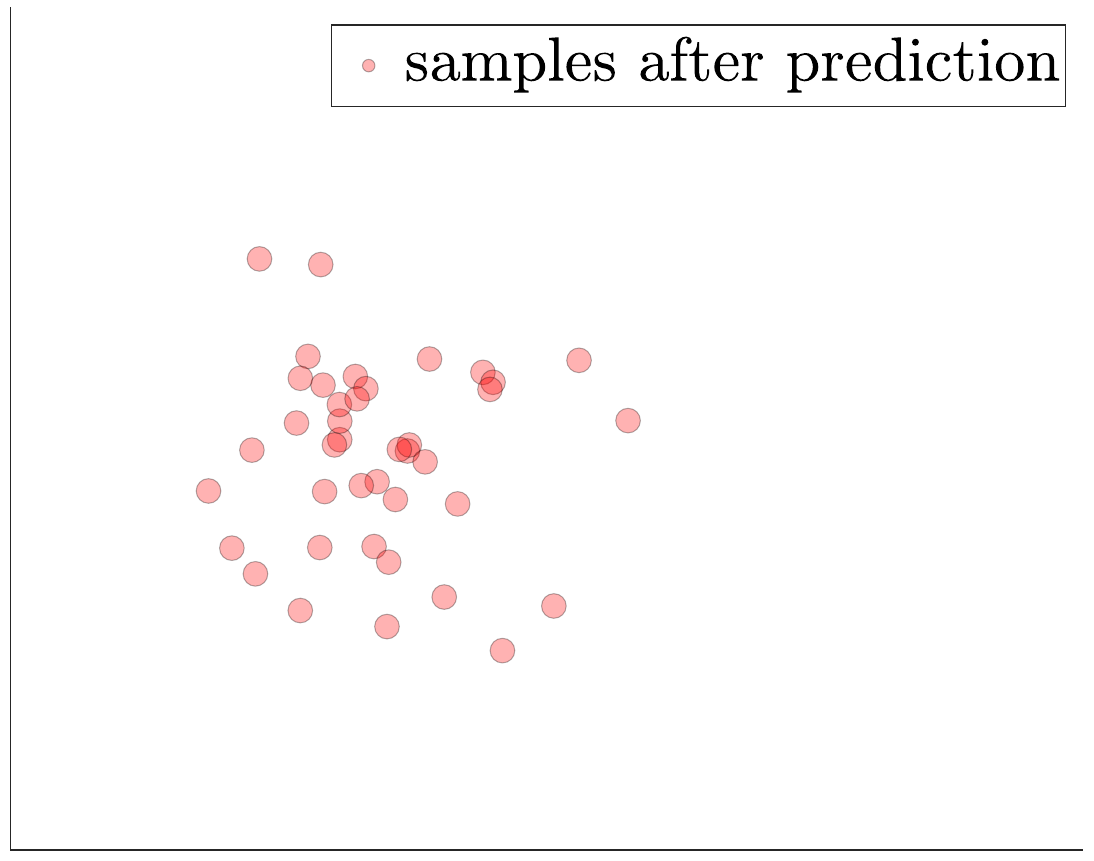}
		\caption*{points after first prediction step \newline }			
	\end{subfigure}
	\hfill
	\begin{subfigure}[b]{0.32\textwidth}
		\centering
		\includegraphics[width=\textwidth]{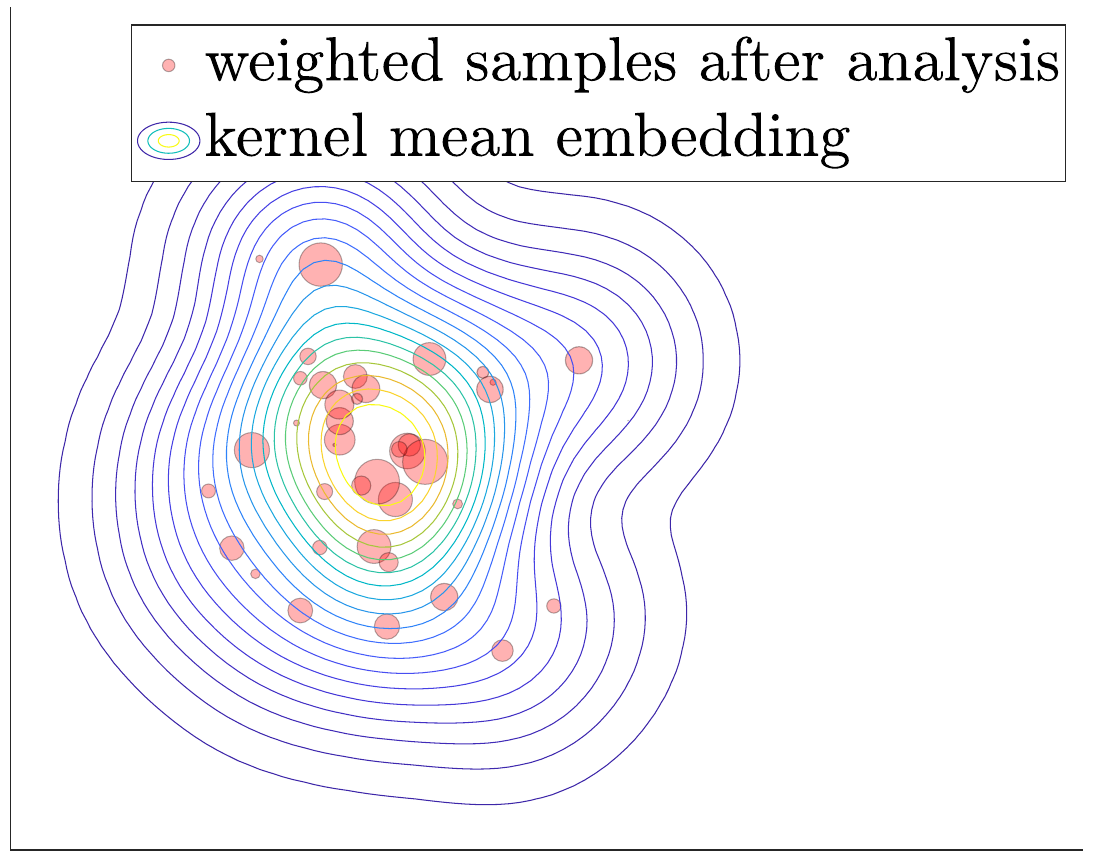}
		\caption*{reweighted points after analysis, with kernel mean embedding}			
	\end{subfigure}
	\hfill
	\begin{subfigure}[b]{0.32\textwidth}
		\centering
		\includegraphics[width=\textwidth]{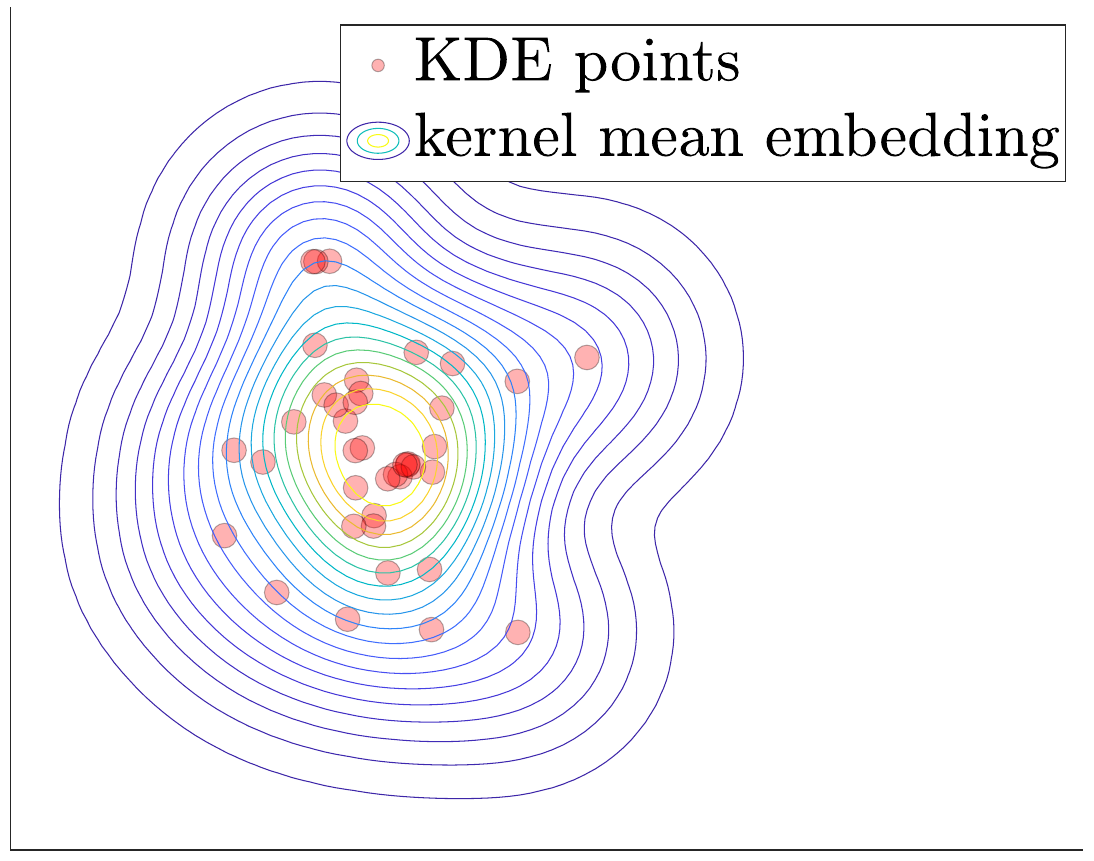}
		\caption*{KDE points with kernel mean embedding}			
	\end{subfigure}
	\vfill
	\begin{subfigure}[b]{0.32\textwidth}
		\centering
		\includegraphics[width=\textwidth]{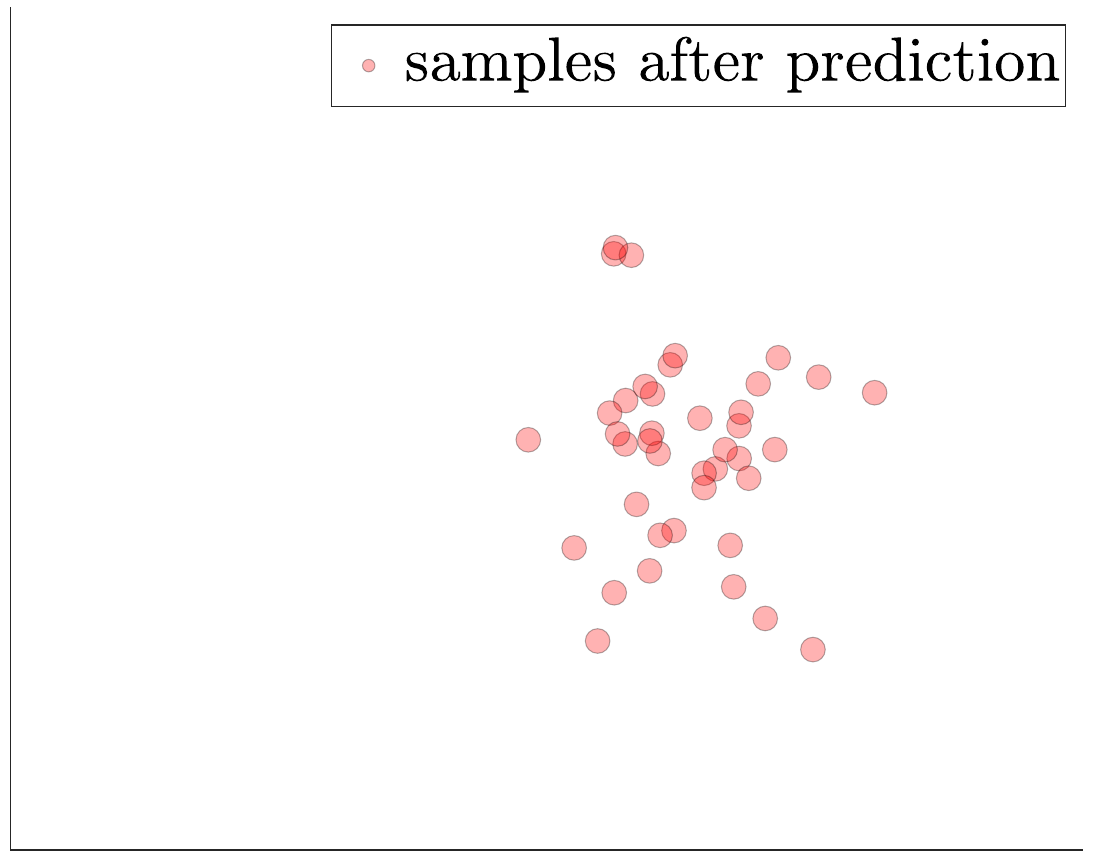}
		\caption*{points after second prediction step}			
	\end{subfigure}
	\hfill	
	\begin{subfigure}[b]{0.32\textwidth}
		\centering
		\includegraphics[width=\textwidth]{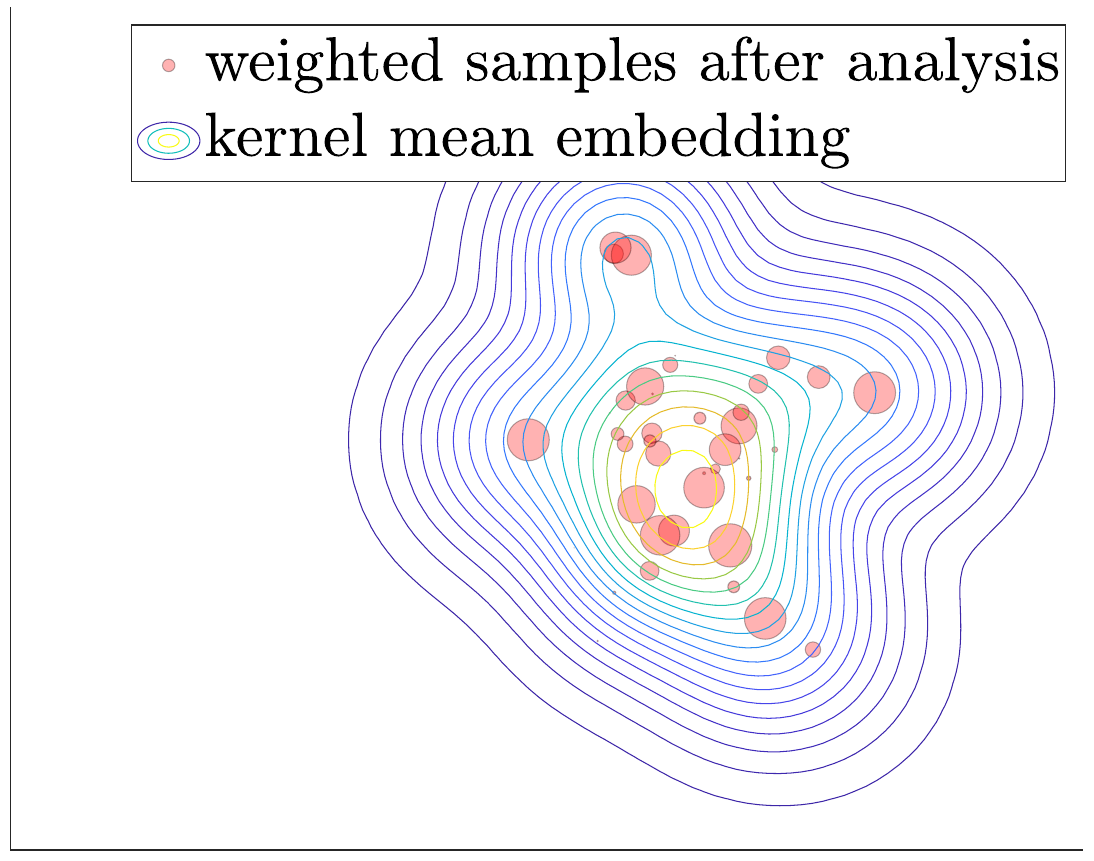}
		\caption*{reweighted points after analysis, with kernel mean embedding}			
	\end{subfigure}
	\hfill
	\begin{subfigure}[b]{0.32\textwidth}
		\centering
		\includegraphics[width=\textwidth]{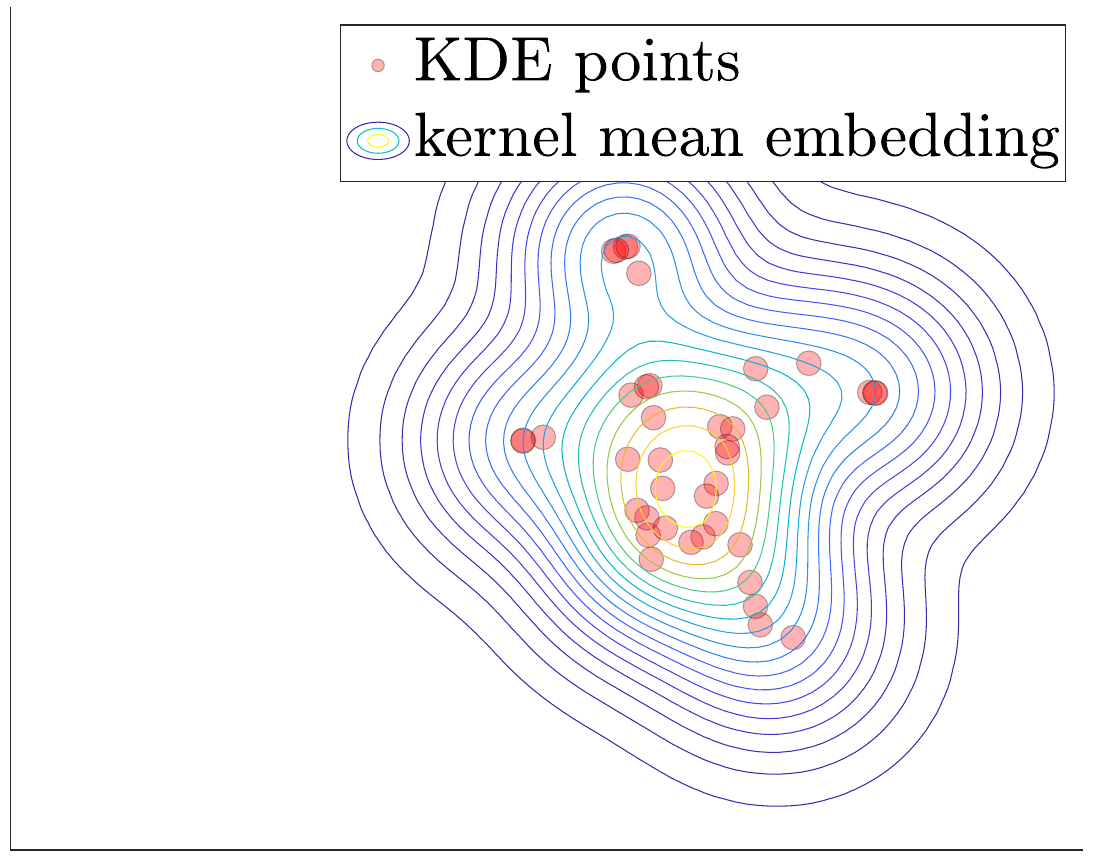}
		\caption*{KDE points with kernel mean embedding}			
	\end{subfigure}
	\caption{Illustration of two steps in SMC for a toy example. The resampling step is carried out by first embedding the weighted point set, followed by outbedding to obtain an unweighted set of KDE points. With only 40 KDE points, the kernel mean embedding is approximated with notable accuracy.}
	\label{fig:SMC_embedding_outbedding}
\end{figure}

Sequential Monte Carlo (SMC) methods are widely used for filtering and particle approximation in dynamical systems, particularly when dealing with non-linear or non-Gaussian models \citep{doucet2000sequential,Doucet2001SMC,delmoral2006smc,Chopin2020SMC}. 
SMC operates through a series of iterations, each consisting of two primary steps: prediction and analysis (also called the update step). 

In the \emph{prediction step}, particles representing the current state distribution are propagated forward using a state transition model. 
This results in a particle approximation of the predicted (prior) distribution.
Next, in the \emph{analysis step}, the particle weights are updated according to the likelihood of the new observations, yielding a weighted particle approximation of the posterior distribution. 

A crucial aspect of SMC is the \emph{resampling step}, which addresses particle weight degeneration. 
Over time, a few particles tend to dominate the weighted sample, reducing the diversity of the particle population.
Resampling reselects particles in proportion to their weights, ensuring that low-weight particles are discarded while high-weight particles are duplicated, thereby maintaining particle diversity \citep[Chapter~9]{Chopin2020SMC}.

The resampling step can be performed using kernel mean embedding and outbedding techniques. 
Instead of directly resampling the particles, one can embed the posterior distribution (represented by the weighted particles) into a reproducing kernel Hilbert space (RKHS) via its kernel mean embedding.
The outbedding step then recovers a set of (unweighted!) resampled particles by approximating the posterior distribution from the embedding by KDE points, as illustrated in \Cref{fig:SMC_embedding_outbedding} for a toy example.
This allows for a principled way to resample particles that effectively captures the underlying distribution, leveraging the properties of KDE points to approximate the posterior distribution effectively (see \Cref{section:Outbedding} for more details).

Such an approach ensures that the resampled particles effectively represent the posterior distribution, offering a reliable and consistent resampling strategy.

%% file: Section_Annealing.tex

\section{Simulated Annealing}
\label{section:Annealing}

As discussed in \Cref{section:Transport_Sampling}, when using a relatively small or moderate number of KDE points, the dynamics \eqref{equ:ODE_with_KDE} may fail to capture certain modes or regions of high $\trd$-probability. This issue is illustrated in \Cref{fig:Rezende_KDE_QMC_vs_MCMC} (left and middle), where the initial distribution struggles to allocate sufficient weight to certain areas. 

We can exacerbate this effect by selecting an initial distribution further from the origin, such as $\simpd = \cN\big( (2,2), \Id_{2} \big)$, as shown in \Cref{fig:Rezende_KDE_annealing} (top). In this case, nearly all of the 40 KDE points are positioned to the right of the origin, and the resulting KDE $\hat{\den}_{t}^{h}$ provides such a poor approximation of the target distribution $\trd$ that even importance sampling, as applied in \Cref{section:Direct_Sampling,section:Transport_Sampling}, is unlikely to sufficiently mitigate this issue with a reasonable number of samples.

For the approaches described in \Cref{section:Convolution_Sampling,section:Outbedding}, where importance sampling is typically not applicable, the consequences could be even more severe.

\begin{figure}[t]
	\centering
	\hfill
	\begin{subfigure}[b]{0.24\textwidth}
		\centering
		\includegraphics[width=\textwidth]{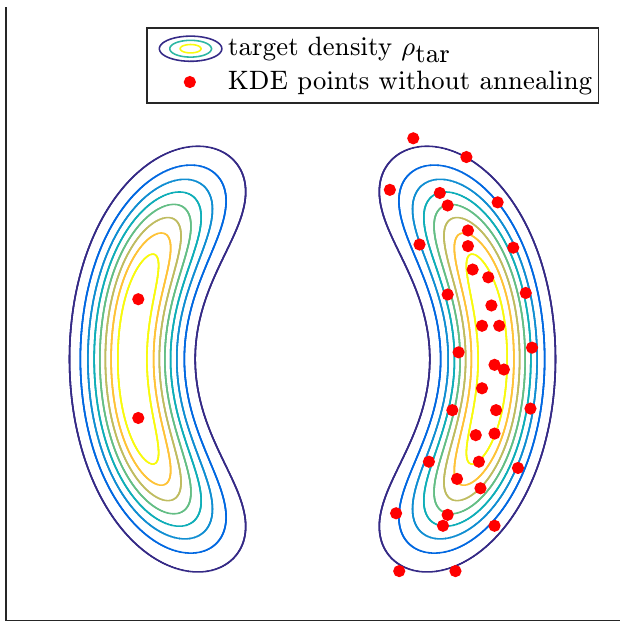}
	\end{subfigure}
	\begin{subfigure}[b]{0.24\textwidth}
		\centering
		\includegraphics[width=\textwidth]{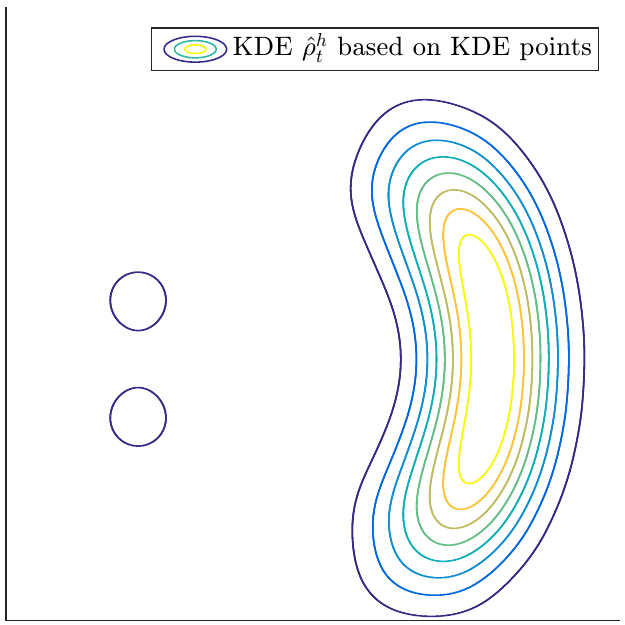}	
	\end{subfigure}
	\vfill	
	\begin{subfigure}[b]{0.24\textwidth}
		\centering
		\includegraphics[width=\textwidth]{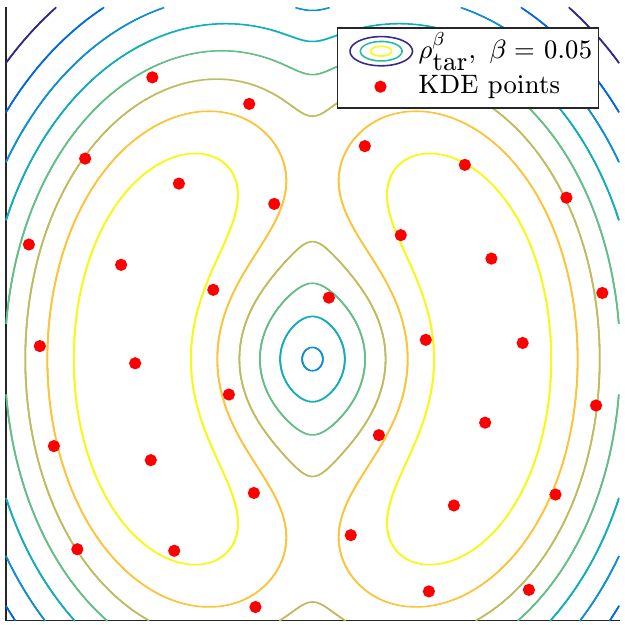}	
	\end{subfigure}
	\begin{subfigure}[b]{0.24\textwidth}
		\centering
		\includegraphics[width=\textwidth]{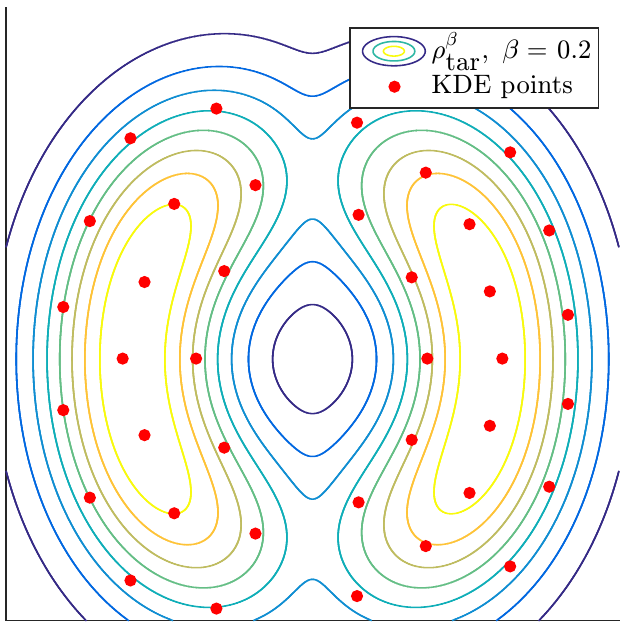}	
	\end{subfigure}
	\begin{subfigure}[b]{0.24\textwidth}
		\centering
		\includegraphics[width=\textwidth]{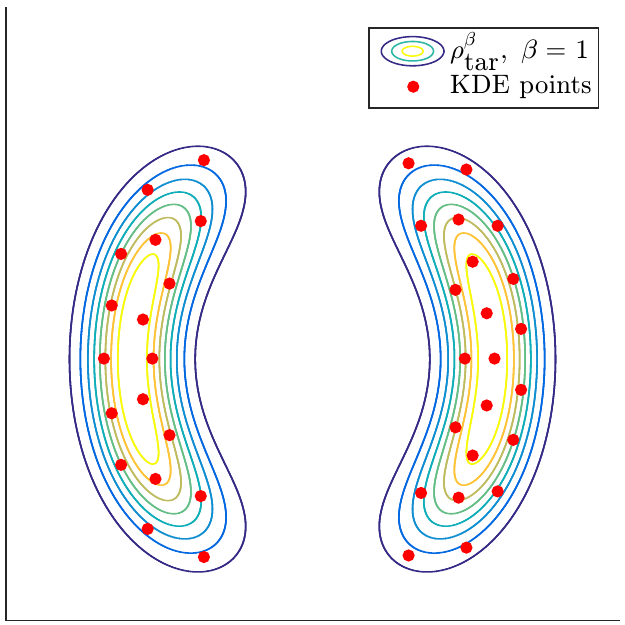}	
	\end{subfigure}
	\begin{subfigure}[b]{0.24\textwidth}
		\centering
		\includegraphics[width=\textwidth]{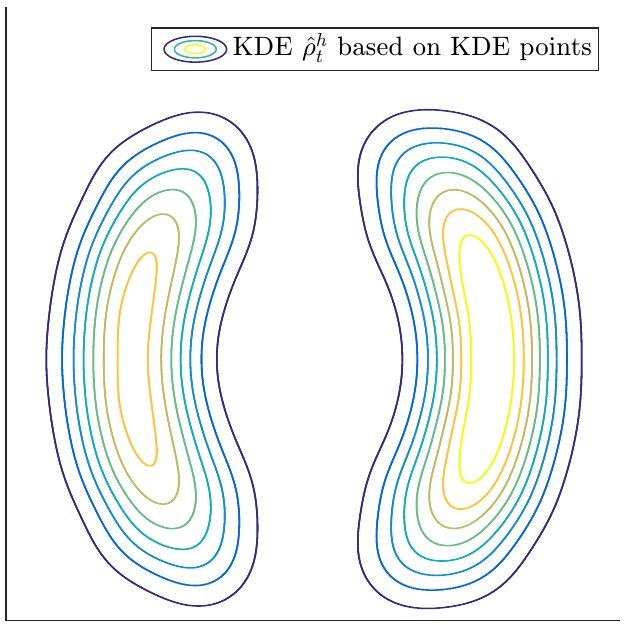}	
	\end{subfigure}
	\caption{\emph{Top:} Without simulated annealing, most KDE points can get trapped in the mode closest to their initial position, leading to a poor approximation $\hat{\den}_{t}^{h}$ of the target (right).
		\emph{Bottom:} Simulated annealing enables the particles to explore the space more extensively, resulting in a better identification of all modes and a more accurate KDE $\hat{\den}_{t}^{h}$ of the target (right).
		Only three inverse temperatures, $\beta_{1} = \nicefrac{1}{20}, \beta_{2} = \nicefrac{1}{5}, \beta_{3} = 1$, were used here.
	}
	\label{fig:Rezende_KDE_annealing}
\end{figure}

As a remedy, we propose using a technique known as \emph{simulated annealing} \citep{kirkpatrick1983simulatedannealing,van1987simulated} in the context of optimization or \emph{tempering} \citep{swendsen1986replica,geyer1995annealing,earl2005parallel} in MCMC settings. This involves artificially increasing the ``temperature'' of the system (parametrized here by the ``inverse temperature'' $\beta \geq 1$) in \eqref{equ:SDE}, and accordingly, in \eqref{equ:ODE_with_KDE}, then gradually cooling the system to the target temperature, $\beta = 1$. This approach allows the system to explore the state space more broadly before converging towards higher-probability modes. 

The tempered SDE can be expressed as
\begin{equation*}
	\mathrm d Y_{\beta}(t)
	=
	\nabla \log \trd (Y_{\beta}(t))\, \mathrm dt + \sqrt{2\beta^{-1}}\, \mathrm dW(t),
	\qquad
	Y(0) = Y_\init \sim \den_\init,
\end{equation*}
with stationary distribution $\den_{\beta,\textup{tar}} \, \propto \, \trd^{\beta}$ and the probability density $\den_{\beta,t}$ evolving as follows:
\begin{equation*}
	\partial_t\den_{\beta,t}
	=
	-\diver(\den_{\beta,t} \nabla\log \trd) + \beta^{-1} \Delta\den_{\beta,t}
	=
	-\diver(\den_{\beta,t} \vel_{\beta,t}^{\FP}),
	\qquad
	\vel_{\beta,t}^{\FP}
	\defeq
	\beta^{-1} \nabla\log \frac{\trd^{\beta}}{\rho_t}.
\end{equation*}
By modifying the velocity field to $\hat{\vel}_{\beta,t}^{h} = \beta^{-1} \nabla \log \frac{\trd^{\beta}}{\hat\den_t^{h}}$ in \eqref{equ:ODE_with_KDE}, and using a sequence of inverse temperatures $\beta_{1} \leq \cdots \leq \beta_{L} = 1$ for which the system evolves over certain time intervals, we enable the dynamics to explore the space more extensively, as shown in \Cref{fig:Rezende_KDE_annealing} (bottom). Alternatively, $\beta(t)$ can be chosen as a function of $t$ converging to $1$, i.e.\ $\beta(t)\nearrow 1$ as $t\to\infty$.

This procedure ensures that KDE points avoid getting ``trapped'' in certain modes, leading to a more even sampling of KDE points and yielding a better approximation $\hat{\den}_{t}^{h}$ of the target $\trd$, as demonstrated in \Cref{fig:Rezende_KDE_annealing} (bottom).

The mathematical analysis of simulated annealing in this context, as well as the discussion of optimal cooling schedules, is beyond the scope of this paper. This approach will likely differ from conventional simulated annealing theory \citep{van1987simulated}, as the dynamics \eqref{equ:ODE_with_KDE} is deterministic rather than stochastic (except for the initial sampling). A simple, though not necessarily optimal, cooling schedule involves running the dynamics at each temperature until convergence is reached (cf.\ the stopping criterion in \Cref{section:technical}), before moving to the next lower temperature.

\subsection{Identifying the need for simulated annealing}

The KDE point method offers a straightforward approach to determine whether certain regions are underrepresented by the sample points. Since $\hat{\rho}_{t}^{h} \defeq \KDE_{\kappa}^{h}[\bsX(t)]$ aims to approximate $\trd$, ideally, the ratio $\frac{\hat{\rho}_{t}^{h}}{\trd}$ should be a constant $C > 0$ (not necessarily $C = 1$, as $\trd$ may not be normalized). In practice, minor deviations in this ratio are acceptable. Slight under- or overrepresentation of certain regions can be corrected by importance sampling. However, if the discrepancy is significant, the ratio will exhibit strong fluctuations, indicating the need for further intervention, which can be identified by evaluating the ratio at the KDE points.

In the example shown in \Cref{fig:Rezende_KDE_annealing} (top), the ratio for the points on the ``right side'' lies between $[0.16, 0.69]$, while for the two points on the left side, it equals $0.025$. This significant imbalance suggests the need for tempering, as shown in \Cref{fig:Rezende_KDE_annealing} (bottom), where the adjusted ratios range from $[0.09, 0.27]$ on the right and $[0.08, 0.21]$ on the left—an acceptable discrepancy that can be handled by importance sampling.

The following strategy summarizes the steps:
\begin{enumerate}[label = (\roman*)]
	\item
	Run the dynamics \eqref{equ:ODE_with_KDE} without tempering.
	\item
	If the fluctuations in the ratios $\frac{\hat{\rho}_{t}^{h}}{\trd}(X_{j}(\infty))$ are significant, for example, if the largest ratio exceeds the smallest by a factor of 10, then underrepresentation is likely, and the temperature should be increased (e.g., tripled).
	\item
	Rerun the dynamics \eqref{equ:ODE_with_KDE} at the adjusted temperature and recheck the ratio fluctuation for the tempered target $\trd^{\beta}$.
	\item
	Repeat this process until the ratio fluctuation is sufficiently reduced.
	\item
	Finally, rerun the dynamics multiple times with the previously used temperatures in decreasing order until the original dynamics \eqref{equ:ODE_with_KDE} is recovered.
\end{enumerate}

If, after this process, the ratio fluctuation remains unsatisfactory, the number of points $J$ should be increased, and possibly the bandwidth $h$ adapted.

%% file: Section_Conclusion.tex

\section{Conclusion and Outlook}
\label{section:Conclusion}

This paper began by focusing on a particle-based variational inference method, then expanded to introduce a novel approach for approximating a target distribution $\trd$ through a mixture distribution $\hat{\den}_{t}^{h}$.
Specifically, this mixture is represented by a kernel density estimate (KDE) constructed from the final particle positions, referred to as \emph{KDE points}.

There are numerous methods for sampling from mixture distributions, including independent and stratified sampling, as well as transporting quasi-Monte Carlo points, sparse grids, and other higher-order point sequences to $\hat{\den}_{t}^{h}$ \citep{Cui2023quasimonte,klebanov2023transporting}. Combining these techniques with (self-normalized) importance sampling creates robust (quasi-) Monte Carlo methods, which have demonstrated superior performance on a bivariate bimodal target distribution.

In addition, we discuss the distribution $\check{\den}^{h}$ of the KDE points themselves, which approximately solves $\check{\den}^{h} \ast \kappa^{h} \approx \trd$, and demonstrate a super-root-$n$ convergence of the corresponding Monte Carlo estimator.
Our method aligns with the framework of kernel herding \citep{chen2010super} and Sequential Bayesian Quadrature \citep{huszar2012herding}, demonstrating comparable performance.
While setups like these are rare in practice, we identified a key area---\emph{kernel mean outbedding}---where this approach proves particularly beneficial, especially in machine learning methods like conditional mean embeddings.
Moreover, we explored how embedding-outbedding strategies could enhance the resampling step in sequential Monte Carlo methods.

A key advantage of the KDE point construction is that it bypasses the need for the normalization constant of $\trd$.
This is particularly significant for Bayesian inverse problems, where this constant is often unknown, highlighting the practical relevance of our approach.

Several important questions remain open. First, while particle convergence has been assumed, a formal investigation of the conditions guaranteeing this would deepen the theoretical foundation of the method.
Additionally, formal proof of the empirical super-root-$n$ convergence claim would be a significant theoretical contribution.
Finally, optimizing the kernel choice, bandwidth, and particle weights (cf.\ \Cref{section:discussion_interacting_particle_dynamics}) offers exciting possibilities for further research, potentially unlocking even greater flexibility and performance improvements.

%% file: Section_AppendixProofs.tex

\section{Proofs}
\label{appendix:Proofs}

\begin{proof}[Proof of \Cref{theorem:Reduction_KL_KDE}]
	By \eqref{equ:ODE_with_KDE}, $(\hat{\den}_{t}^{h})_{t\geq 0}$ satisfies the continuity equation
	\begin{align}
		\label{equ:perturbed_continuity_equation}
		\partial_{t} \hat{\den}_{t}^{h}
		&=
		- \frac{1}{J}\sum_{j=1}^{J} \nabla\kappa^{h}(\quark - X_j(t))^{\intercal} \, \hat{v}_{t}^{h}(X_{j}(t))
		=
		- \diver( \hat{\den}_{t}^{h} w_{t}^{h}),
		\\
		\label{equ:perturbed_velocity_field}
		w_{t}^{h}
		&\defeq
		\frac{1}{J\hat{\den}_{t}^{h}} \sum_{j=1}^{J} \kappa^{h}(\quark - X_j(t)) \, \hat{v}_{t}^{h}(X_{j}(t)).
	\end{align}
	Since $\hat{v}_{t}^{h}$ is continuous there exists $\varepsilon > 0$ such that, for each $j=1,\dots,J$ and $x\in \Ball{0}{\varepsilon}$,
	\[
	\hat{v}_{t}^{h}(X_{j}(t))^{\intercal} \, \hat{v}_{t}^{h}(X_j(t) + x)
	\geq
	\tfrac{1}{2}\norm{\hat{v}_{t}^{h}(X_{j}(t))}^{2}.
	\]
	If $\hat{v}_{t}^{h}(X_{j}(t)) = 0$ for each $j=1,\dots,J$, then there is nothing to show.
	Otherwise,
	\[
	A \defeq J^{-1} \sum_{j=1}^{J} \norm{\hat{v}_{t}^{h}(X_{j}(t))}^{2} > 0,
	\qquad
	B \defeq J^{-1} \sum_{j=1}^{J} \norm{\hat{v}_{t}^{h}(X_{j}(t))} > 0
	\]
	are strictly positive and, by \eqref{equ:technical_decay_condition}, there exists $h_{0} > 0$ such that, for any $0<h \leq h_{0}$ and $j=1,\dots,J$,
	\[
	\int_{\Ball{0}{\varepsilon}} \kappa^{h}(x)\, \mathrm{d}x
	\geq
	\tfrac{1}{2},
	\qquad
	\int_{\bR^{d}\setminus \Ball{0}{\varepsilon}} \kappa^{h}(x) \, \norm{\hat{v}_{0}^{h}(X_{j} + x)}\, \mathrm{d} x
	\leq
	\tfrac{A}{4B}.
	\]
	Setting $h_{\ast} \defeq h_{0}$, \Cref{prop:KL_change} implies for every $0<h\leq h_{\ast}$
	\begin{align*}
		-\tfrac{\mathrm d}{\mathrm dt}&\big|_{t=0} \dkl(\hat{\den}_{t}^{h}\|\trd)
		=
		\innerprod{w_{0}^{h}}{\hat{\vel}_{0}^{h}}_{L^2_{\hat{\den}_{0}^{h}}}
		\\
		&=
		\frac{1}{J} \sum_{j=1}^{J} \int_{\bR^{d}} \kappa^{h}(x - X_{j}) \, \hat{v}_{0}^{h}(X_{j})^{\intercal} \, \hat{v}_{0}^{h}(x)\, \mathrm{d} x
		\\
		&\geq
		\frac{1}{J} \sum_{j=1}^{J}
		\int_{\Ball{0}{\varepsilon}} \kappa^{h}(x) \, \hat{v}_{0}^{h}(X_{j})^{\intercal} \, \hat{v}_{0}^{h}(X_{j} + x)\, \mathrm{d} x		
		-
		\norm{\hat{v}_{0}^{h}(X_{j})} \int_{\bR^{d}\setminus \Ball{0}{\varepsilon}} \kappa^{h}(x) \, \norm{\hat{v}_{0}^{h}(X_{j} + x)}\, \mathrm{d} x
		\\
		&\geq
		\frac{1}{J} \sum_{j=1}^{J}
		\frac{\norm{\hat{v}_{0}^{h}(X_{j})}^{2}}{4} - \frac{A\norm{\hat{v}_{0}^{h}(X_{j})}}{4B}
		\\
		&=
		\frac{A}{4} - \frac{AB}{4B}
		\\
		&=
		0.
	\end{align*}	
\end{proof}

\begin{proof}[Proof of \Cref{theorem:Reduction_KL_KDE_infinite_ensemble}]
	If $\overline{\vel}_{0}^{h} = 0$ Lebesgue-almost everywhere, then there is nothing to show.
	Otherwise, since $\rho$ and $\overline{\vel}_{0}^{h} \in L_{\rho}^{2}$ are continuous, the values $A \defeq \norm{\overline{\vel}_{0}^{h}}_{L_{\rho}^{2}}^{2} > 0$ and $B \defeq \norm{\overline{\vel}_{0}^{h}}_{L_{\rho}^{1}} > 0$ are finite and strictly positive.
	By \ref{item:Cross_Correlation_in_L2}, there exist $h_{2}>0$ and a compact set $\cK \subseteq \bR^{d}$ such that, for any $0<h\leq h_{2}$,
	\[
	\Absval{\int_{\bR^{d}\setminus \cK} \rho(y) \, \overline{\vel}_{0}^{h}(y)^{\intercal} (\overline{\vel}_{0}^{h}\star \kappa^{h})(y)\, \mathrm dy}
	\leq
	\int_{\bR^{d}\setminus \cK} \rho(y) \, (\norm{\overline{\vel}_{0}^{h}(y)}^{2} +  \norm{(\overline{\vel}_{0}^{h}\star \kappa^{h})(y)}^{2})\, \mathrm dy
	\leq
	\frac{A}{8}.
	\]
	For $y\in\cK$ let
	\[
	E(y)
	\defeq
	\sup \big\{ \varepsilon>0 \mid \overline{\vel}_{0}^{h}(y+x)^{\intercal} \overline{\vel}_{0}^{h}(y) \geq \tfrac{1}{2} \norm{\overline{\vel}_{0}^{h}(y)}^{2} \text{ for all } x\in \Ball{0}{\varepsilon} \big\}.
	\]
	Since $\overline{\vel}_{0}^{h}$ is continuous, $E$ is continuous and strictly positive and thereby, since $\cK$ is compact, $E$ attains its minimum $\varepsilon \defeq \min_{y\in\cK} E(y) > 0$. Further, for $y\in\cK$ let
	\[
	H(y)
	\defeq
	\sup \bigg\{ h_{1} > 0 \biggm\vert \int_{\bR^{d}\setminus \Ball{0}{\varepsilon}} \kappa^{h}(x) \, \norm{\overline{v}_{0}^{h}(y + x)}\, \mathrm{d} x \leq \tfrac{A}{8B} \ \forall 0 < h \leq h_{1};\ \int_{\Ball{0}{\varepsilon}} \kappa^{h_{1}}(x)\, \mathrm{d}x \geq \tfrac{1}{2} \bigg\}.
	\]
	Since $\kappa$ is a continuous probability density and using \ref{item:compact_set}, $H$ is continuous and strictly positive and thereby, since $\cK$ is compact, $H$ attains its minimum $h_{3} \defeq \min_{y\in\cK} H(y) > 0$.
	Hence, for each $y\in\cK$ and $0 < h \leq h_{3}$,
	\begin{align*}
		\int_{\bR^{d}} \kappa^{h}(x-y)&\, \overline{\vel}_{0}^{h}(x)^{\intercal} \overline{\vel}_{0}^{h}(y)\, \mathrm{d}x
		\\
		&\geq
		\int_{\Ball{y}{\varepsilon}} \kappa^{h}(x-y)\, \overline{\vel}_{0}^{h}(x)^{\intercal} \overline{\vel}_{0}^{h}(y)\, \mathrm{d}x - 
		\Absval{\int_{\bR^{d}\setminus \Ball{y}{\varepsilon}} \kappa^{h}(x-y)\, \overline{\vel}_{0}^{h}(x)^{\intercal} \overline{\vel}_{0}^{h}(y)\, \mathrm{d}x}
		\\
		&\geq
		\int_{\Ball{0}{\varepsilon}} \kappa^{h}(x)\, \underbrace{\overline{\vel}_{0}^{h}(y+x)^{\intercal} \overline{\vel}_{0}^{h}(y)}_{\geq \tfrac{1}{2} \norm{\overline{\vel}_{0}^{h}(y)}^{2}}\, \mathrm{d}x -
		\norm{\overline{\vel}_{0}^{h}(y)} \underbrace{\int_{\bR^{d}\setminus \Ball{0}{\varepsilon}} \kappa^{h}(x) \, \norm{\overline{v}_{0}^{h}(y + x)}\, \mathrm{d} x}_{\leq \tfrac{A}{8B}}
		\\
		&\geq
		\frac{\norm{\overline{\vel}_{0}^{h}(y)}^{2}}{4} - \frac{A \, \norm{\overline{\vel}_{0}^{h}(y)}}{8B}.
	\end{align*}
	By \eqref{equ:TransportEquation}, the family of densities $(\rho_{t})_{t\geq 0}$ satisfies the continuity equation
	$\partial_t\den_t = -\diver(\den_t \overline{\vel}_{t}^{h}), \ \den_0 = \den$.
	By applying Leibniz integral rule and differentiation rules for convolutions (using \ref{item:L_1_condition}), it follows that the family of convoluted densities $(\overline{\den}_{t}^{h})_{t\geq 0}$ satisfies the continuity equation for sufficiently small times $t\geq 0$,
	\begin{align*}
		\partial_t \overline{\den}_{t}^{h}
		&=
		(\partial_t \den_{t}) \ast \kappa^{h}
		=
		- \big( \diver(\den_{t} \overline{\vel}_{t}^{h}) \big) \ast \kappa^{h}
		=
		-\diver\big((\den_{t} \overline{\vel}_{t}^{h})  \ast \kappa^{h} \big)
		=
		-\diver(\overline{\den}_{t}^{h} \overline{w}_t),
		\\
		\overline{w}_{t}
		&\defeq
		\frac{(\den_{t} \overline{\vel}_{t}^{h})\ast \kappa^{h}}{\overline{\den}_{t}^{h}}.
	\end{align*}
	By \Cref{prop:KL_change} and setting $h_{\ast} \defeq \min(h_{2},h_{3}) > 0$, it follows for $0 < h \leq h_{\ast}$ that,
	\begin{align*}
		-\tfrac{\mathrm{d}}{\mathrm{d}t} \big|_{t=0} &\dkl(\overline{\rho}_{t}^{h} \| \trd)
		=
		\innerprod{\overline{w}_{0}}{\overline{\vel}_{0}^{h}}_{L_{\overline{\rho}_{0}}^{2}}
		\\
		&=
		\int_{\bR^{d}} \big( (\den \overline{\vel}_{0}^{h})\ast \kappa^{h} \big)(x)^{\intercal} \, \overline{\vel}_{0}^{h}(x) \, \mathrm{d}x
		\\
		&=
		\int_{\bR^{d}} \int_{\bR^{d}} \kappa^{h}(x-y)\den(y) \overline{\vel}_{0}^{h}(y)^{\intercal}\overline{\vel}_{0}^{h}(x) \, \mathrm{d}y  \, \mathrm{d}x
		\\
		&\geq
		\int_{\cK} \rho(y) \int_{\bR^{d}} \kappa^{h}(x-y) \overline{\vel}_{0}^{h}(y)^{\intercal}\overline{\vel}_{0}^{h}(x) \, \mathrm{d}x\, \mathrm dy
		-
		\Absval{\int_{\bR^{d}\setminus \cK} \rho(y) \, \overline{\vel}_{0}^{h}(y)^{\intercal} (\overline{\vel}_{0}^{h}\star \kappa^{h})(y)\, \mathrm dy}
		\\
		&\geq
		\int_{\cK} \rho(y) \left(\frac{\norm{\overline{\vel}_{0}^{h}(y)}^{2}}{4} - \frac{A \, \norm{\overline{\vel}_{0}^{h}(y)}}{8B}\right)\, \mathrm dy - \frac{A}{8}
		\\
		&=
		\frac{A}{4} - \frac{A B}{8B} - \frac{A}{8}
		\\
		&=
		0.
	\end{align*}
\end{proof}